\documentclass[sigconf,nonacm]{aamas}

\usepackage{balance} 
\usepackage{enumitem}
\usepackage{algorithm}
\usepackage{algpseudocode}
\usepackage{siunitx}
\usepackage{microtype}

\setlist{nosep,topsep=2pt,partopsep=0pt,itemsep=1pt,parsep=0pt}

\allowdisplaybreaks

\doi{SEHD8631}

\makeatletter
\gdef\@copyrightpermission{
  \begin{minipage}{0.2\columnwidth}
   \href{https://creativecommons.org/licenses/by/4.0/}{\includegraphics[width=0.90\textwidth]{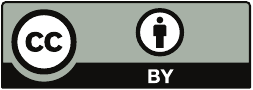}}
  \end{minipage}\hfill
  \begin{minipage}{0.8\columnwidth}
   \href{https://creativecommons.org/licenses/by/4.0/}{This work is licensed under a Creative Commons Attribution International 4.0 License.}
  \end{minipage}
  \vspace{5pt}
}
\makeatother

\setcopyright{ifaamas}
\acmConference[AAMAS '26]{Proc.\@ of the 25th International Conference
on Autonomous Agents and Multiagent Systems (AAMAS 2026)}{May 25 -- 29, 2026}
{Paphos, Cyprus}{C.~Amato, L.~Dennis, V.~Mascardi, J.~Thangarajah (eds.)}
\copyrightyear{2026}
\acmYear{2026}
\acmDOI{}
\acmPrice{}
\acmISBN{}

\acmSubmissionID{115}

\title[ReCORS]{Conformal Reachability for Safe Control in Unknown Environments}

\author{Xinhang Ma}
\affiliation{
  \institution{Washington University in St. Louis}
  \city{St. Louis}
  \country{USA}}
\email{m.owen@wustl.edu}

\author{Junlin Wu}
\affiliation{
  \institution{Washington University in St. Louis}
  \city{St. Louis}
  \country{USA}}
\email{junlin.wu@wustl.edu}

\author{Yiannis Kantaros}
\affiliation{
  \institution{Washington University in St. Louis}
  \city{St. Louis}
  \country{USA}}
\email{ioannisk@wustl.edu}

\author{Yevgeniy Vorobeychik}
\affiliation{
  \institution{Washington University in St. Louis}
  \city{St. Louis}
  \country{USA}}
\email{yvorobeychik@wustl.edu}

\begin{abstract}
Designing provably safe control is a core problem in trustworthy autonomy. 
However, most prior work in this regard assumes either that the system dynamics are known or deterministic, or that the state and action space are finite, significantly limiting application scope.
We address this limitation by developing a probabilistic verification framework for unknown dynamical systems which combines conformal prediction with reachability analysis.
In particular, we use conformal prediction to obtain valid uncertainty intervals for the unknown dynamics at each time step, with reachability then verifying whether safety is maintained within the conformal uncertainty bounds.
Next, we develop an algorithmic approach for training control policies that optimize nominal reward while also maximizing the planning horizon with sound probabilistic safety guarantees.
We evaluate the proposed approach in seven safe control settings spanning four domains---cartpole, lane following, drone control, and safe navigation---for both affine and nonlinear safety specifications.
Our experiments show that the policies we learn achieve the strongest provable safety guarantees while still maintaining high average reward.
\end{abstract}


\newcommand{\BibTeX}{\rm B\kern-.05em{\sc i\kern-.025em b}\kern-.08em\TeX}

\newcommand{\owen}[1]{\textcolor{orange}{Owen: #1}}

\newtheorem{theorem}{Theorem}

\newcommand{\mc}[1]{\mathcal{#1}}

\begin{document}


\pagestyle{fancy}
\fancyhead{}


\maketitle 


\section{Introduction}

Trustworthy control of autonomous systems in dynamic environments is a fundamental problem in a broad range of domains including autonomous driving, drone navigation, and robotic manipulation, among others.
While reinforcement learning techniques have made significant strides in advancing autonomy capabilities in these domains~\cite{azar2021drone,han2023survey,kiran2021deep,xu2024omnidrones,zhang2021reinforcement}, their trustworthiness is opportunistic, insofar as they tend to combine performance objectives and safety considerations into a a single reward function, and achieve high reward accounting for safety only on average.
In safety critical environments, such as self-driving, more rigorous safety guarantees are commonly expected.

An extensive literature has emerged to address this issue, addressing both the problem of verifying safety~\cite{ames2019control,ames2016control,ivanov2019verisig,jackson2020safety,tran2019safety,zhang2024seev, zhang2023exact}, as well as synthesizing safe control policies (using both formal methods and learning techniques)~\cite{ganai2023iterative,gu2022review,ma2022conservative,wu2024verified,yu2022reachability}.
However, verification techniques typically assume known or deterministic dynamics, and many suffer from severe scalability limitations when dynamics are nonlinear, particularly for complex (e.g., neural network) control policies.
Similarly, synthesis techniques, including those using safe reinforcement learning methods, either make very strong assumptions on the dynamical system  (for example, assuming that it is known and deterministic)~\cite{wu2024verified}, assume a finite state or action space~\cite{chow2018lyapunov,jackson2020safety}, or achieve safety only approximately~\cite{ganai2023iterative,ma2021feasible,xu2021crpo,yu2022reachability,yu2023safe} (for example, by relying on learned estimates of the probability that safety is violated~\cite{ganai2023iterative}), on average (as in the constrained Markov decision process (CMDP) framework)~\cite{altman2021constrained,tessler2018reward,yang2020projection}, or in the limit~\cite{ganai2023iterative,yu2022reachability}.
Moreover, while a number of approaches have been proposed to deal with unknown dynamical systems, 
most still rely on restrictive assumptions to obtain provable safety guarantees, 
such as
assuming that dynamics are deterministic~\cite{jackson2020safety,wang2024providing,zhou2022neural},
constraining approximation models to specific forms such as Gaussian process regression~\cite{jackson2020safety},
or requiring Lipschitz continuity~\cite{zhou2022neural}.

We present a novel approach for learning policies in unknown dynamical systems  that approximately maximize rewards with rigorous probabilistic safety guarantees.
Our approach combines model-based RL---we learn a Markovian approximation of the dynamics---with conformal safety analysis (CSA) to obtain safety guarantees, and model-free RL to optimize cumulative rewards.
By jointly training the dynamics model and policy using a loss that incentivizes safety improvements over a target decision horizon, our approach facilitates a finer-granularity tradeoff between the nature of (probabilistic) safety guarantees and rewards than prior art.
Moreover, the use of conformal prediction (CP) as a key building block allows us to make no assumptions on the nature of the underlying dynamics (which we allow, for example, to be stochastic) or the model architecture used for learning it.
Finally, our method provides sound probabilistic safety guarantees by combining three technical elements:
valid uncertainty bounds for dynamics prediction,
finite-horizon reachability analysis,
and probabilistic tail bounds with respect to the unknown initial state distribution.
Our CSA approach builds on recent literature combining conformal prediction with reachability analysis~\cite{chakraborty2025safety,hashemi2023data,hashemi2024statistical,lindemann2024formal,muthali2023multi}, but unlike most prior efforts we make no smoothness assumptions on the underlying dynamics, allow verification for an arbitrary horizon post surrogate model training, and develop novel training methods for both surrogate dynamics and policy that obtain strong safety guarantees.

We evaluate the proposed approach experimentally on seven safe control experiment settings, comparing it to state of the art safe RL baselines.
Our results show that our approach achieves state of the art safety guarantees while maintaining high reward.

In summary, our main contributions are as follows:
\begin{enumerate}[nosep]
    \item \emph{Conformal Safety Analysis (CSA)} for safe reinforcement learning in unknown dynamical systems. In particular, our CSA framework combines conformal prediction with reachability analysis for learning-based controllers in systems with \emph{unknown dynamics}, which we augment with distribution-free tail bounds for an unknown initial state distribution to obtain rigorous probabilistic safety guarantees.
    \item A novel framework for learning provably safe policies that (approximately) maximize cumulative reward subject to safety constraints. Our approach integrates model-free RL and model-based reachability estimation with quantified uncertainties to enforce hard safety constraints. 
    \item Comprehensive experiments across seven environments, including both linear and nonlinear safety specifications, showing that our approach achieves high rewards while providing verifiable probabilistic safety guarantees.
\end{enumerate}

\section{Related Work}\label{sec:related_work}
Ensuring safety in dynamical systems 
has been classically modeled with Constrained Markov Decision Processes (CMDPs) \cite{altman2021constrained}, which seek to maximize expected rewards while constraining cumulative costs below a predefined threshold. Common approaches for solving CMDPs include Lagrangian and penalty methods, as well as constrained policy optimization~\citep{achiam2017constrained, jayant2022model, ma2022conservative, so2023solving, stooke2020responsive, yu2022reachability}. While these methods are effective at balancing reward and cost, they primarily enforce empirical safety and lack formal safety guarantees.

To address this limitation, extensive literature has also emerged for obtaining provably safe or stable control, particularly for policies that are represented as neural networks
.
Many of these rely on control barrier functions, which are themselves difficult to synthesize or verify at scale~\cite{cheng2019end,dawson2022safe,wang2023enforcing,zhao2023probabilistic}.
Others proposed reinforcement learning-based training algorithms to obtain policies with provable safety and stability guarantees~\cite{wu2024verified,zhou2022neural}, but these scale poorly with the number of states, and either assume known or deterministic dynamics.
Data-driven methods such as~\cite{fan2019formal, fan2019data, fan2017dryvr, salamati2022data} use simulation-based reachability or scenario optimization, but are often conservative and require large sample sizes to achieve high confidence. 
Our approach, in contrast, uses conformal prediction to provide statistically calibrated, distribution-free uncertainty quantification.
This enables scalable and model-agnostic verifications under any model errors. 

A number of recent approaches combine reachability analysis with conformal prediction, as we do~\cite{hashemi2024statistical,lindemann2024formal,muthali2023multi}, with several also addressing the issue of unknown dynamics~\cite{chakraborty2025safety,hashemi2023data}.
However, these either assume smooth dynamics or a fixed decision horizon $K$.
In contrast, our approach trains Markovian surrogate models that can subsequently be applied for any $K$, and we make no assumptions on the underlying dynamics.
Moreover, unlike prior work, we train verifiable combinations of surrogate dynamics and policies.

Another class of approaches for verifying safety makes use of Hamilton-Jacobi (HJ) reachability techniques.
These compute exact reachable/avoidable sets via PDE solutions, but do not scale well and typically assume knowledge of system dynamics~\cite{fisac2019bridging,kochdumper2023provably, selim2022safe, yu2023safe}.
Yu et al.~\cite{yu2022reachability} proposed an approach that guarantees safety when agent is in the feasible set,
while Ganai et al.~\cite{ganai2023iterative} extended it to general stochastic settings.
However, our framework is fundamentally different:
instead of solving the computationally expensive PDEs,
we adopt a discrete-time finite-step forward reachability formulation that scales better with system complexity.
Moreover, our use of conformal prediction enables us to apply our approach without making any assumptions about unknown dynamics.








\section{Preliminaries}\label{sec:preliminaries}

\subsection{Deep Reinforcement Learning}
\label{S:drl}
As our approach includes features of model-free reinforcement learning (RL), we provide a brief overview of the main concepts.
Generally, such approaches train a critic $V_{\theta_c}$ (e.g., value or action-value function) and an actor (policy) $\pi_{\theta_a}$, in some cases including also an entropy term $\mc{H}$ that encourages exploration, with the general loss function of the form
\[
\mathcal{J}(\theta_c,\theta_a) = \mathbb{E}_{s,a,r,s'}[\mathcal{L}_c(\theta_c) + \alpha_1 \mathcal{L}_a(\theta_a) + \alpha_2 \mathcal{H}[\pi_{\theta_a}]]
\]
for observed sequences of states and actions $s,a,s'$ and rewards $r$.
In actor-critic methods in particular (such as PPO~\cite{schulman2017proximal}), a common approach is to define an advantage function 
$A(s,a) = r + \gamma V_{\theta_c}(s') - V_{\theta_c}(s)$, 
where $\gamma$ is the discount factor,
and $r$ and $s'$ stem from taking action $a$, with $\mc{L}_c(\theta_c) = A(s,a)^2$ and $\mc{L}_a(\theta_a)$ is an algorithm-specific function of $A(s,a)$.

\subsection{Conformal Prediction}
\label{S:cp}

Conformal prediction (CP) is a distribution-free and model-agnostic statistical uncertainty quantification technique.
As a result, it has come to be particularly useful in supervised learning, 
where it can provide valid statistical guarantees for complex models such as deep neural networks.
Of particular interest in our context is \emph{regression learning}, which we use to learn a model of system dynamics in a state space that is embedded in a Euclidean space, and for which CP provides 
valid prediction intervals~\cite{vovk2005algorithmic}. 
Here we provide an overview of conformal prediction in the form of split conformal prediction, which is most widely used. For a more comprehensive review and more examples, we refer the readers to~\cite{angelopoulos2021gentle}.

Given a general input $x$ and output $y$, 
first define the score function $s(x,y) \in \mathbb{R}$ where larger scores reflect larger uncertainties (larger errors between $x$ and $y$).
Let $\alpha \in (0,1)$ be a user-specified miscoverage level that represents the desired probability of error in the uncertainty estimates.
We compute $\hat{q}$ as the $\frac{\lceil (n+1)(1-\alpha) \rceil}{n}$th quantile of the scores of a calibration set $\{s_i =s(x_i,y_i)\}$ for $1 \leq i \leq n$.
Finally, for a new input $x_{\text{test}}$, we use $\hat{q}$ to provide the uncertainty interval for the regression model $\hat{f}$ as
\begin{equation}
    \mathcal{C}(x_{\text{test}}) = [\hat{f}(x_{\text{test}})-\hat{q}, \hat{f}(x_{\text{test}})+\hat{q}]
\end{equation}
Crucially, the interval $\mathcal{C}(x_{\text{test}})$ achieves 
a $\mathbf{P}( y_{\text{test}} \in \mathcal{C}(x_{\text{test}})) \geq 1-\alpha$ 
coverage guarantee~\cite{vovk2005algorithmic}.


Because standard split conformal prediction assumes exchangeable nonconformity scores $s_0,\ldots,s_n$, it warrants adaptation when applied in time-series prediction.
A naive example is to use a union bound over the prediction horizon to connect per-step and full-trajectory bounds.
However, this can result in bounds that are overly conservative.

To obtain tighter conformal trajectory bounds, we make use of the approach by Cleaveland et al.~\cite{cleaveland2024conformal}, which learns a trajectory-level nonconformity score that can directly be optimized for efficiency.
Specifically, one can define
\[
s = \max (w_1 s(x_1,y_1),\ldots,w_Ts(x_T,y_T))
\]
where $w_1,\ldots,w_T \geq 0$ are parameters.
Cleaveland et al formulate
the problem of finding the parameters $w$ as the following optimization problem:
\begin{align*}
    &\min_{w \ge 0} \quad\text{Quantile}(\{s^{(1)},\ldots,s^{(n)}\},1-\alpha) \\
    &\text{s.t.} \quad s^{(i)}=\max(w_1s_1^{(i)},\ldots,w_Ts_T^{(i)})\ \forall \ i; \quad \sum_{j=1}^Tw_j=1,
\end{align*}
where $s^{(i)}$ is the score for the $i$-th trajectory in a held-out dataset.
The details of solving this optimization problem efficiently can be found in Section 4 of \cite{cleaveland2024conformal}.
Importantly, conformal region computed using this nonconformity score with optimized $w$ parameters maintains the desired coverage guarantee (Theorem 3 in \cite{cleaveland2024conformal}).

\section{Model}\label{sec:model}
We consider a safety-constrained Markov decision process (MDP) (SMDP), defined by a tuple $\mathcal{M} := \langle \mathcal{S}, \mathcal{A}, P, \rho, r, h, K, \gamma \rangle$, 
where $\mathcal{S} \subseteq \mathbb{R}^n$ is the state space, 
$\mathcal{A}$ the action space, 
$P: \mathcal{S} \times \mathcal{A} \rightarrow \mc{P}(\mathcal{S})$ the transition distribution (i.e. the environment dynamics) which maps current state and action to a distribution $\mc{P}$ over states, 
$\rho$ the initial state distribution,
$r: \mathcal{S} \times \mathcal{A} \rightarrow \mathbb{R}$  the reward function assigning a real value reward to each state-action pair, 
$h: \mathcal{S} \rightarrow \mathbb{R}^{m}_+$ is the cost function assigning a real vector to a state representing its associated $m$-dimensional safety value (e.g., corresponding to $m$ safety constraints, such as $m$ distinct obstacles to avoid), $K$ the decision horizon, 
and $\gamma \in (0,1)$ is the discount factor.
Thus, $h(s) \le 0$ pointwise implies that the state $s \in \mc{S}$ is \emph{safe} (safety cost is non-negative) while $h(s) > 0$ means that it is \emph{unsafe}.
Let $\mathcal{S}_s \subseteq \mc{S} := \{s \in \mathcal{S} : h(s)\le 0\}$ denote the set of safe states, while $\mathcal{S}_u:=\{s \in \mathcal{S}: h(s)>0\}$ is the set of unsafe states.
Correspondingly, the SMDP imposes a safety constraint that the system dynamics remain safe until the safety horizon $K$. 
Formally, for any trajectory of states $\tau = \{s_0,\ldots,s_K\}$, we wish to ensure that $s_t \in \mc{S}_s$ for all $1 \le t \le K$.
We assume that the support of $\rho$ is over $\mc{S}_s$ (that is, all initial states are safe).
If this constraint holds, we say that the system is \emph{safe}.
We remark that the safety constraints in SMDPs are distinct from the typical nature of constraints in \emph{constrained MDPS (CMDPs)}, which accrue (discounted) costs additively over time analogously to rewards in the objective.
Many algorithmic frameworks for CMDPs leverage this structure (for example, using Lagrangian methods to transform it into an MDP).
In contrast, the safety constraints in SMDPs are non-additive, and such algorithmic approaches cannot be leveraged without the use of additive proxies.

Let $\pi:\mc{S} \rightarrow \mc{A}$ be a deterministic Markov stationary policy, and $\Pi$ the set of all such policies.
As is conventional in MDPs, we define the state-conditional \emph{policy value function} $V^\pi(s)$ as the discounted sum of rewards, i.e., 
\[
V^\pi(s)=\mathbb{E}\left[\sum_{t=0}^K\gamma^t r(s_t,a_t)|s_0=s, s_t \sim P(s_{t-1},a_{t-1}), a_t = \pi(s_t)\right].
\]
Informally, our goal is to find a policy $\pi$ that maximizes $V^\pi(s)$ while ensuring that the system is safe.
However, we suppose that we do not know the system dynamics $P$ or reward $r$.
While we can use conventional reinforcement learning approaches to learn a near-optimal policy in an MDP without safety considerations, guaranteeing perfect safety is impractical.
Instead, we aim to achieve safety \emph{with high probability}.
Define 
\[
\Sigma^\pi = \Pr\left[s_t \in \mathcal{S}_s, \forall t \in \{0,1,\ldots,K\}\right]
\]
where $s_0 \sim \rho, s_t \sim P(s_{t-1},a_{t-1}), a_t = \pi(s_t)$,
that is, $\Sigma^\pi$ is the probability that the entire trajectory of horizon $K$ generated by a policy $\pi$ is safe, accounting for all relevant uncertainty (which we will make more precise in Section~\ref{S:cra}).

\emph{Our goal is to solve the following constrained optimization problem}:
\begin{equation}
\label{E:problem}
\max_{\pi \in \Pi} \quad \mathbb{E}_{s \sim \rho} [V^\pi(s)] \quad \mathrm{s.t.: } \quad \Sigma^\pi \ge 1-\zeta
\end{equation}
for a target safety tolerance level $\zeta$.
Crucially, while we can accept approximately optimal policies $\pi$ in terms of the objective, we wish to obtain \emph{verified} safety at the desired tolerance, since these often capture considerations that have high stakes.
For example, rewards would naturally capture efficiency, such as getting to the destination quickly, whereas safety constraints focus on avoiding traffic accidents that may potentially lead to fatalities.

Solving Problem~\eqref{E:problem} entails two considerations: 1) obtaining sound probabilistic safety guarantees for a given deterministic policy $\pi$ as expressed in the safety constraint (which we call \emph{verification}) and 2) synthesizing (or learning) a policy that is approximately optimal while also satisfying the safety constraint for the target decision horizon $K$ and tolerance $\zeta$.
Moreover, we must accomplish both goals \emph{without knowing the true dynamics $P$ and reward $r$.}
We tackle these problems next.

\section{Solution Approach}\label{sec:approach}







In this section, 
we introduce our approach, \emph{ReCORS (\underline{R}eachability with \underline{CO}nformal uncertainty sets for \underline{R}einforcement learning with \underline{S}afety guarantees)}, 
which involves three main components:

\begin{enumerate}
    \item Approximating the true system dynamics $P$ with a neural network $\hat{f}_\beta$ that is optimized for accuracy in safety-critical state dimensions. 

    \item A novel integration of conformal prediction with reachability analysis performed using the learned dynamics model $\hat{f}_\beta$.
    The resulting \emph{conformal safety analysis (CSA)} enables verification of whether the system is likely to remain in safe set $\mathcal{S}_s$ over horizon $K$ with a pre-defined confidence level $(1-\alpha)$. 

    \item A fully differentiable pipeline allowing end-to-end gradient-based learning of a parametric control policy $\pi_\beta$: 
    we use a neural network $\eta_\omega$ to approximate the conformal prediction intervals, 
    allowing the safety objective (derived from $\hat{f}_\beta$ and $\eta_\omega$) to be directly integrated as a differentiable objective term in any gradient-based policy optimization algorithm. 
\end{enumerate}
Notably, our approach is a novel blend of model-based and model-free reinforcement learning.
In the model-based component, we learn the model of the dynamics (but not the rewards), which is pivotal for training with probabilistic safety guarantees and associated verification.
Our learning pipeline then leverages a combination of model-free RL (e.g., PPO) to maximize rewards with the model-based component focusing on optimizing for provable safety.
We detail the components of our approach next.

\subsection{Learning System Dynamics with Safety Specifications}

A key component of our approach is training a parametric discrete-time dynamics model $\hat{f}_\beta(\hat{s},a)$ with parameters $\beta$, which \emph{deterministically} maps state $\hat{s}$ and action $a$ to predicted next state $\hat{s}'$.
Thus, our \emph{learned} approximate dynamical system model has the form
\begin{equation}
\label{E:approx_dyn}
\hat{s}_{t+1} = \hat{f}_\beta(\hat{s}_t,a_t),
\end{equation}
where $\hat{s}_t$ and $a_t$ are the (previously predicted) state and action taken at time step $t$, respectively, and $\hat{s}_{t+1}$ is the predicted next state.
Notably, even though the underlying dynamics may be stochastic, while $\hat{f}$ is deterministic, our use of conformal prediction will account for the concomitant prediction uncertainty (essentially combining both the prediction error and any underlying uncertainty of the dynamics itself); we discuss this further below.

To train the dynamical system model, we collect a set of ground-truth trajectories $\{\tilde{\tau}_l\}$ where $\tilde{\tau}_l = (s_{l0},a_{l0},s_{l1},\ldots,s_{lK-1},a_{lK-1},s_{lK})$ by sampling initial states $s_0$ i.i.d.~from $\rho$.
This induces a dataset of $N$ tuples $\mc{D} = \{(s_i,a_i,s_i')\}_{i=1}^N$ that we can use for training.
Notably, we train the dynamics $\hat{f}$ in two phases.
First, we pre-train $\hat{f}$ by generating the trajectory data by sampling actions $a_t$ uniformly at random.
Subsequently, we keep fine-tuning the dynamics model as part of the proposed RL approach (see below), where in each training epoch we resample the trajectory data using the latest policy $\pi_\beta$ obtained during the RL training.

The loss function we use for training $\hat{f}_\beta$ is a variant of $\ell_2$ loss.
Since a major concern in our context is safety, we modify the standard $\ell_2$ loss to account for safety considerations by adding weights that amplify the importance of safety-critical dimensions of the state space.
Specifically, let $\phi_j(h(s)) > 0$ be a positive scalar that measures the safety criticality of dimension $j$ with respect to the safety function $h(s)$.
In particular, suppose without loss of generality that the $i$th safety constraint has the form $h_i(s) = \bar{h}_i(s) + b_i$ for some scalar $b_i$.
We propose the following measure of safety criticality of dimension $j$: 
\[
\phi_j(h(s)) = \prod_{i=1}^m\left(\frac{|\partial \bar{h}_i(s)/\partial s_j|}{|b_i|}+1\right),
\]
where $|\cdot|$ is an absolution value.
For example, if $h(s) = As - b$, 
$\partial \bar{h}_i(s)/\partial s_j = A_{ij}$.
Intuitively, $\phi_j$ introduces a higher weight for dimensions which are more likely to induce safety violations.

Next, define the safety weight of dimension $j$ as $w_j = \phi_j(h(s))$.
Our training loss then becomes
\begin{equation}\label{eq:dynamics_loss}
    \mathcal{L}_{\text{dyn}}(\beta) = \frac{1}{N}\sum_{i=1}^{N} \|\hat{f}_\beta(s_i,a_i) - s_i'\|_W,
\end{equation}
where $W$ is a diagonal matrix with $w_j$s along the diagonal, and $\|x\|_W = x^T W x$.

The expected effect of adding such safety weights as part of the loss function used to train the dynamics is two-fold.
On the one hand, it will incentivize training to focus on making better predictions along the dimensions that are more safety-critical.
On the other hand, this may increase overall (unweighted) prediction error, which would in turn lead to reduced reward.
This tradeoff is deliberate: while we aim to maximize total reward, we prioritize \emph{verified safety}.
Notably, our approach is heuristic, to be sure, but as our ablations in Section~\ref{S:exp_ablations} demonstrate, it has the intended effect of improving verified safety.


\subsection{Conformal Safety Analysis}
\label{S:cra}

Consider the approximate dynamical system model in Equation~\eqref{E:approx_dyn} with a known initial state $s_0$.
We can represent this as a collection of $K$ predictions, each corresponding to a model that predicts the $t$th state in the sequence for $t \in [1,K]$ given an initial state $s_0$ and a policy $\pi_\theta$, that is,
\begin{align}
\hat{s}_{t} = \hat{f}_\beta^{(t)}(s_0;\pi_\theta)
\end{align}
for $t > 0$, 
where $\hat{s}_1=\hat{f}_\beta^{(1)}(s_0;\pi_\theta) = \hat{f}_\beta(s_0,\pi_\theta(s_0))$, $\hat{s}_2=\hat{f}_\beta^{(2)}(s_0;\pi_\theta) = \hat{f}_\beta(\hat{s}_1,\pi_\theta(\hat{s}_1))$, and so on, with $a_t = \pi_\theta(\hat{s}_{t})$.
In other words, given a policy $\pi_\theta$, the predicted dynamical system model $\hat{f}_\beta$ effectively gives rise to $K$ separate state predictions, each with $s_0$ as input.
Moreover, if $s_0$ are sampled i.i.d.~from the initial state distribution $\rho$,
we can apply conformal prediction bounds for each step $t$ to obtain
\begin{equation}
\label{E:conformal_k}
    \mathbf{P} [\|s_{t} - \hat{f}_\beta^{(t)}(s_0, \pi_\theta)\| \le \eta_t] \geq 1-\alpha_t,
\end{equation}
where $s_{t}$ is a true state generated by the actual system dynamics after $t$ steps starting at $s_0$, $\alpha_t$ 
the error tolerance rate, 
and $\eta_t$ the conformity score obtained using a calibration dataset to ensure the validity of~\eqref{E:conformal_k}.

However, obtaining conformal bounds for each $t$ is insufficient, since our goal is to obtain such bounds over $K$-step trajectories, and predictions can be arbitrarily correlated across time steps.
Specifically, let $S_t(s_0;\pi_\theta) = \{s | \|s_{t} - \hat{f}_\beta^{(t)}(s_0, \pi_\theta)\| \le \eta_t\}$ denote the set of states that can be reached after exactly $t$ steps given the conformity score $\eta_t$ if we start at $s_0$ and follow the policy $\pi_\theta$.
Let $S(s_0;\pi_\theta) = \cup_{t \in [1,K]} S_t(s_0;\pi_\theta)$.
We can then observe that if $S(s_0;\pi_\theta) \cap \mc{S}_u = \emptyset$, the policy $\pi_\theta$ is safe, modulo the predictive uncertainty of $\hat{f}$.
Our goal is to obtain this claim rigorously for a given policy $\pi_\theta$, with a target probability $1-\zeta$.
One way to do this is to simply apply a union bound to the step-wise conformal bounds $\alpha_t$ over all steps $t$.
However, this is often too lose.
As an alternative, we additionally consider a more efficient time-series conformal bound proposed by Cleaveland et al.~\cite{cleaveland2024conformal}.
Armed with a time-series-level conformal bound, we obtain the following result, which borrows the probabilistic reasoning with respect to the initial state distribution from semi-probabilistic verification~\cite{ma2025learning} and gives us a sound probabilistic safety bound, accounting for uncertainty about (a) initial state, (b) imperfect predictions by $\hat{f}_\beta$, and (c) samples of finite datasets of $s_0$ used to obtain empirical safety.


\begin{theorem}\label{thm:hoeffding_bound}
Let $\{s_0^i\}_{i=1}^N$ be i.i.d.\ samples from $\rho$, independent of the calibration data used to construct the conformal bounds. 
Suppose the conformal procedure provides a \emph{trajectory-level} coverage guarantee that for every fixed initial state $s_0$,
\(
\Pr\!\big(S(s_0;\pi_\theta)\cap\mathcal S_u=\emptyset\big)\;\ge\; 1-\alpha.
\)
Let $V=\{s_0^i : S(s_0^i;\pi_\theta)\cap \mathcal S_u = \emptyset\}$. 
Then, for any $\delta\in(0,1)$,
\[
\Sigma^{\pi_{\theta}}
\;\ge\; \left(\frac{|V|}{N} \;-\; \sqrt{\tfrac{1}{2N}\log\tfrac{2}{\delta}}\right)(1-\alpha)(1-\delta).
\]
\end{theorem}
\begin{proof}[Proof Sketch]
For any fixed $s_0$, the conformal coverage guarantee implies that
\(
\Pr\!\big(S(s_0;\pi_\theta)\cap \mathcal S_u=\emptyset\big)\;\ge\; 1-\alpha.
\)
Now let $x = \Pr_{s\sim\rho}[S(s;\pi_\theta)\cap \mathcal S_u=\emptyset]$ 
and $\hat x = |V|/N$. 
Applying Hoeffding’s inequality, we obtain
\[
\Pr\!\left(x \ge \hat x - \sqrt{\tfrac{1}{2N}\log\tfrac{2}{\delta}}\right) \;\ge\; 1-\delta,
\]
where the probability is over the i.i.d.~draws of $N$ initial states $s_0$.
Combining this with the conformal trajectory guarantee ($1-\alpha$) yields the result.
\end{proof}


We refer to this general approach as \emph{conformal safety analysis (CSA)}.
A key practical benefit of CSA is that the conformal procedure produces per–time-step reachable sets 
\(\,S_t(s_0;\pi_\theta)\)
so the safety condition can be checked separately for each \(t\).
Consequently, our verification problem reduces to checking
\[
S_t(s_0;\pi_\theta)\cap\mathcal S_u=\emptyset\quad\text{for each }t,
\]
Modern neural network verification tools, such as \((\alpha,\beta)\)-CROWN \cite{xu2021fast, shi2025neural}, can handle nonlinear models (e.g., ReLU or tanh networks) and thus general nonlinear safety specifications.
Moreover, many practically relevant cases involve relatively simple constraints: 
for example, safety specification often decomposes across coordinates (e.g., axis-aligned box constraints), in which case it suffices to check safety independently for each dimension.

\subsection{Training with Differentiable Conformal and Safety Loss}

CSA provides us with a safety verification approach.
Our key concern, however, is to \emph{train} policies and dynamic models to enable us to achieve safety.
This entails two challenges: first, how do we design a loss function that promotes smaller conformal uncertainty sets and higher confidence (which, in combination, allow us to achieve tighter reachability bounds), and second, how do we design a loss function that specifically promotes safety?
We tackle these in turn next.

\paragraph{Conformal Loss}

While traditional conformal prediction provides statistical guarantees,
it is not differentiable and thus incompatible with end-to-end gradient-based training.
To address this issue, we make use of an approach proposed by~\cite{bai2022efficient} that enables us to obtain a differentiable approximation of conformal sets using a neural network, which takes state and action as an input, and predicts the corresponding conformity threshold $\eta$.
We denote the neural network by $\eta_\omega(s,a)$, with $\eta_{\omega,j}(s,a)$ the threshold predicted for state variable $j$.
This neural network is trained to balance two objectives:
1) minimizing the prediction set size (``efficiency''), and 2) ensuring that the empirical coverage level is above a target threshold.

To formalize, we define the efficiency loss (which promotes smaller conformal sets) as
\[
\mathcal{L}_{\text{eff}}(\omega) = \frac{1}{N}\sum_{(s,a) \in \mathcal{D}_e} \left(\prod_{j=1}^n \eta_{\omega,j}(s,a)\right)
\]
where $n$ is the state dimension and $\mc{D}_e$ is a dataset of tuples $(s,a,s')$ obtained from trajectories $\tilde{\tau}$ sampled in epoch $e$ of training using the latest policy $\pi_\theta$, where we fix the datasets size $N = |\mc{D}_e|$.

Next, the coverage loss penalizes violations of the target coverage, and is defined as
\[
\mathcal{L}_{\text{cov}}(\beta,\omega) = \max(0, 1-\alpha-C_{\textit{emp}}),
\]
where $C_{\textit{emp}}$ is the empirical coverage defined as
\[
C_{\textit{emp}}(\beta,\omega) = \frac{1}{N}\sum_{(s, a, s') \in \mathcal{D}_e} \mathbf{1}[||s' - \hat{f}_{\beta}(s, a)|| \le \eta_{\omega}(s,a)].
\]
However, note that empirical coverage is not differentiable with respect to the parameters $(\beta,\omega)$.
To address this issue, we propose the following differentiable proxy:
\[
\hat{C}_{emp}(\beta,\omega) = \frac{1}{N} \sum_{(s,a,s') \in \mathcal{D}_e} \sigma\left(1-\max_j \frac{|\hat{f}_{\beta}(s, a) - s'_j|}{\max\{\eta_{\omega,j}(s,a),c\}}\right)
\]
where $c$ is a small constant ensuring that the loss is well-behaved and $\sigma(\cdot)$ is a sigmoid that smoothly approximates the indicator.
To unpack, observe that for the max norm,
$\max_j \frac{|\hat{f}_{\beta}(s, a) - s'_j|}{\eta_{\omega,j}(s,a)} \le 1$ 
iff 
$||s' - \hat{f}_{\beta}(s, a)|| \le \eta_{\omega}(s,a)$, i.e., it is equivalent to $s'$ lying inside the conformal set.

Finally, we combine these into conformal loss as
\[
\mathcal{L}_{\text{conf}}(\beta,\omega) = \max_{\lambda \ge 0}\;\big(\mathcal{L}_{\text{eff}}(\omega) + \lambda \mathcal{L}_{\text{cov}}(\beta,\omega)\big),
\]
where $\lambda$ is a Lagrange multiplier periodically updated during training using projected gradient ascent.



\paragraph{Safety Loss}

Next, we consider the safety aspect of the problem explicitly.
In particular, we construct a safety loss $\mc{L}_{\text{safety}}(\theta,\omega)$ that takes the learned dynamics as given, and consists of two components: one which maintains safety over the entire sequence of time steps, and the second which explicitly encourages a policy to improve the safety score between initial and $K$th step, with $K$ increasing over training iterations somewhat echoing curriculum learning methods for learning safe control~\cite{wu2024verified}.
Note that the safety loss is explicitly also a function of the conformal interval parameters $\eta_\omega$, since these have a direct impact on verified safety.
Additionally, safety loss is computed on a pre-collected initial state dataset $\mc{D}_0 = \{s_0^i\}_{i=1}^{N_0}$, $s_0^i \sim \rho$.
By evaluating safety consistently on this fixed set of initial states rather than randomly sampled states each epoch, we ensure the safety loss signal remains stable and meaningful throughout training, avoiding noise from sampling variance that could obscure true safety improvements.

Define a \emph{safety score function}  with respect to a set of states $S$, $g(S) = \max_{s \in S} h(s)$, quantifying the maximal safety violation of any $s \in S$.
The first safety loss component captures the maximal safety loss over a target safety horizon $K_e$ in training epoch $e$:
\[
\mathcal{L}_{\text{safety}}^{\text{max}}(\theta,\omega) = \max_{\substack{t \in [1..K_e];\\ s_0 \in \mathcal{D}_0}} g(\hat{f}_\beta^{(t)}(s_0;\pi_\theta) \pm \eta_\omega(s_t,a_t)).
\]

The second part of the safety loss aims to incentivize learning to \emph{improve} (or reduce the degradation) of safety over the longer-term evolution of the dynamical system (without consideration of transient terms), which we formalize as
\begin{align*}
    \mathcal{L}_{\text{safety}}^{\text{improve}}(\theta,\omega) = \frac{1}{N_0 K_e}\sum_{s_0 \in \mathcal{D}_0} [g(\hat{f}^{(1)}_\beta(s_0; \pi_\theta) \pm \eta_{\omega}(s_0,a_0)) \\ - g(\hat{f}^{(K_e)}_\beta(s_0; \pi_\theta) \pm \eta_{\omega}(s_{K_e},a_{K_e}))]
\end{align*}

The full safety loss is then a weighted combination of the two terms:
\[
\mc{L}_{\text{safety}}(\theta,\omega) = w_{\text{max}}\mc{L}_{\text{safety}}^{\text{max}}(\theta,\omega) + w_{\text{improve}} \mc{L}_{\text{safety}}^{\text{improve}}(\theta,\omega),
\]
with the weights $w_{\text{max}}$ and $w_{\text{improve}}$ specified as hyperparameters.


\subsection{Putting Everything Together: The ReCORS Algorithm}

The complete ReCORS algorithm proceeds as follows.
First, it pre-trains the dynamics model using a dataset generated by selecting random actions, collects the initial state dataset for safety loss computation, and initializes the policy (actor), critic (as in conventional RL), and conformal uncertainty neural networks.
In each training epoch, the algorithm then 
(i) collects a new dataset of trajectories with the current policy, 
(ii) fine-tunes the dynamics model using this dataset with $\mathcal{L}_{\text{dyn}}$, 
and (iii) updates the remaining components. 
As in standard actor–critic methods, the critic parameters $\theta_c$ are trained separately using the conventional RL loss $\mathcal{L}_{\text{RL}}(\theta_c,\theta)$. 
The policy parameters $\theta$ and conformal network parameters $\omega$ are then optimized jointly using a combined loss
\[
\mc{L}(\theta_c,\theta,\beta,\omega) = w_1 \mc{L}_{\text{RL}}(\theta_c,\theta) + w_2 \mc{L}_{\text{conf}}(\beta,\omega)  + w_3 \mc{L}_{\text{safety}}(\theta,\omega),
\]
where $w_1$, $w_2$, $w_3$ are hyperparameters, $\theta_c$ are the parameters of the critic, $\theta$ parameters of the actor (policy), $\beta$ parameters of the dynamics model, and $\omega$ parameters of the conformal uncertainty estimation network.

An additional feature of the ReCORS algorithm is a form of \emph{curriculum learning}: we initially begin with a smaller safety horizon $K_0$, and incrementally increase the horizon whenever provable safety is achieved or after every $E$ epochs.
Algorithm~\ref{algo:main} formalizes this discussion.

\begin{algorithm}
\caption{ReCORS}\label{algo:main}
\hspace*{\algorithmicindent} \textbf{Input} 
pretrained dynamics model $\beta$, pre-collected initial state dataset $\mc{D}_0$, initial safety horizon $K_0$, epoch counter $E$, total number of epochs $T$
\begin{algorithmic}[1]
\State Initialize $\theta_c, \theta, \omega$
\For{epoch $e = 0, 1, 2, \ldots, T$}
    \State Collect trajectories and induced dataset $\mathcal{D}_e$ using the current policy $\pi_{\theta_e}$, or random actions if $e=0$
    
    \State Dynamics update: $\beta_{e+1} \leftarrow \arg\min_{\beta} \mathcal{L}_{\text{dyn}}(\beta_e)$.

    \State Critic update: $\theta_{c,e+1} \leftarrow \arg\min_{\theta_c} \mathcal{L}_{\text{RL}}(\theta_{c,e})$

    \State Actor and uncertainty update:
    \begin{align*}
        (\theta_{e+1}, \omega_{e+1}) \leftarrow \arg\min_{\theta,\omega} 
        \big( &w_1 \mathcal{L}_{\text{RL}}(\theta_e) 
        \\+ &w_2 \mathcal{L}_{\text{conf}}(\beta_{e+1},\omega_e) 
        \\+ &w_3 \mathcal{L}_{\text{safety}}(\theta_e,\omega_e) \big)
    \end{align*}
    
    \State Evaluate coverage: 
    \[
    C_{\textit{emp}} = \tfrac{1}{N}\sum_{(s,a,s') \in \mathcal{D}_e} 
    \mathbf{1}\big[s' \in \hat{f}_{\beta_{e+1}}(s,a) \pm \eta_{\omega_{e+1}}(s,a)\big]
    \]
    
    \If{($\mc{L}_{\text{safety}}^{\text{max}} < 0$ and $C_{\textit{emp}} \geq 1-\alpha$) or every $E$ epochs}
        \State $K_{e+1} = K_e + 1$
    \Else
        \State $K_{e+1} = K_e$
    \EndIf
\EndFor
\end{algorithmic}
\end{algorithm}

\section{Experiments}\label{sec:experiments}
\subsection{Experiment Setup}

We evaluate ReCORS in six environments with linear safety constraints and one environment with nonlinear constraints.
First, we designed a 2D Quadrotor environment with a nonlinear safe distance constraint for demonstrating ReCORS in the most general setting.
And for the linear cases, 
the first four, Cartpole, Lane Following, and 2D and 3D Quadrotor control, have hard state space safety constraints with binary cost signals, where the episode terminates with any violation. 
The other two environments, CarGoal and HalfCheetah, are from Safety-Gymaniusm \cite{ji2023safety}, which have more complex underlying system dynamics, richer cost structures, and are higher dimensional (e.g., $72$ dimensions for CarGoal).
Detailed experiment descriptions are provided in the Appendix.

We compare to five SOTA baselines: 
1) Lagrangian-based Proximal Policy Optimization (PPOLag)~\cite{ray2019benchmarking}; 
2) Projection-based Constrained Policy Optimization (PCPO)~\cite{yang2020projection}; 
3) Constraint-Rectified Policy Optimization (CRPO)~\cite{xu2021crpo}; 
4) Penalized Proximal Policy Optimization (P3O)~\cite{zhang2022penalized}; and 
5) Reachability Estimation for Safe Policy Optimization
(RESPO)~\cite{ganai2023iterative}. 
For consistency and fair comparison, all but RESPO were implemented based on~\cite{ji2024omnisafe} and all used similar hyperparameters.
Details of specific settings and hyperparameters are provided in the Appendix.

We use both the time series conformal prediction method proposed by Cleaveland et al.~\cite{cleaveland2024conformal} and the union bound applied to the step-wise split conformal prediction.
For both methods, we set the trajectory level miscoverage level $\alpha$ to $0.1$.
Our calibration set contains $1000$ trajectories,
where, for the times series method, we use $100$ for optimization of the $w$ parameters and the rest to compute the actual prediction regions.
We then perform forward reachability analysis on $2000$ new trajectories (verification set) using the calibrated thresholds.
After obtaining the empirical proportions of trajectories that are safe at each time step, we apply Theorem~\ref{thm:hoeffding_bound} (with $\delta=0.05$) to compute the lower bound safety probabilities with respect to the entire initial state space.

\begin{figure*}[h!]
    \centering
    \includegraphics[width=0.65\linewidth]{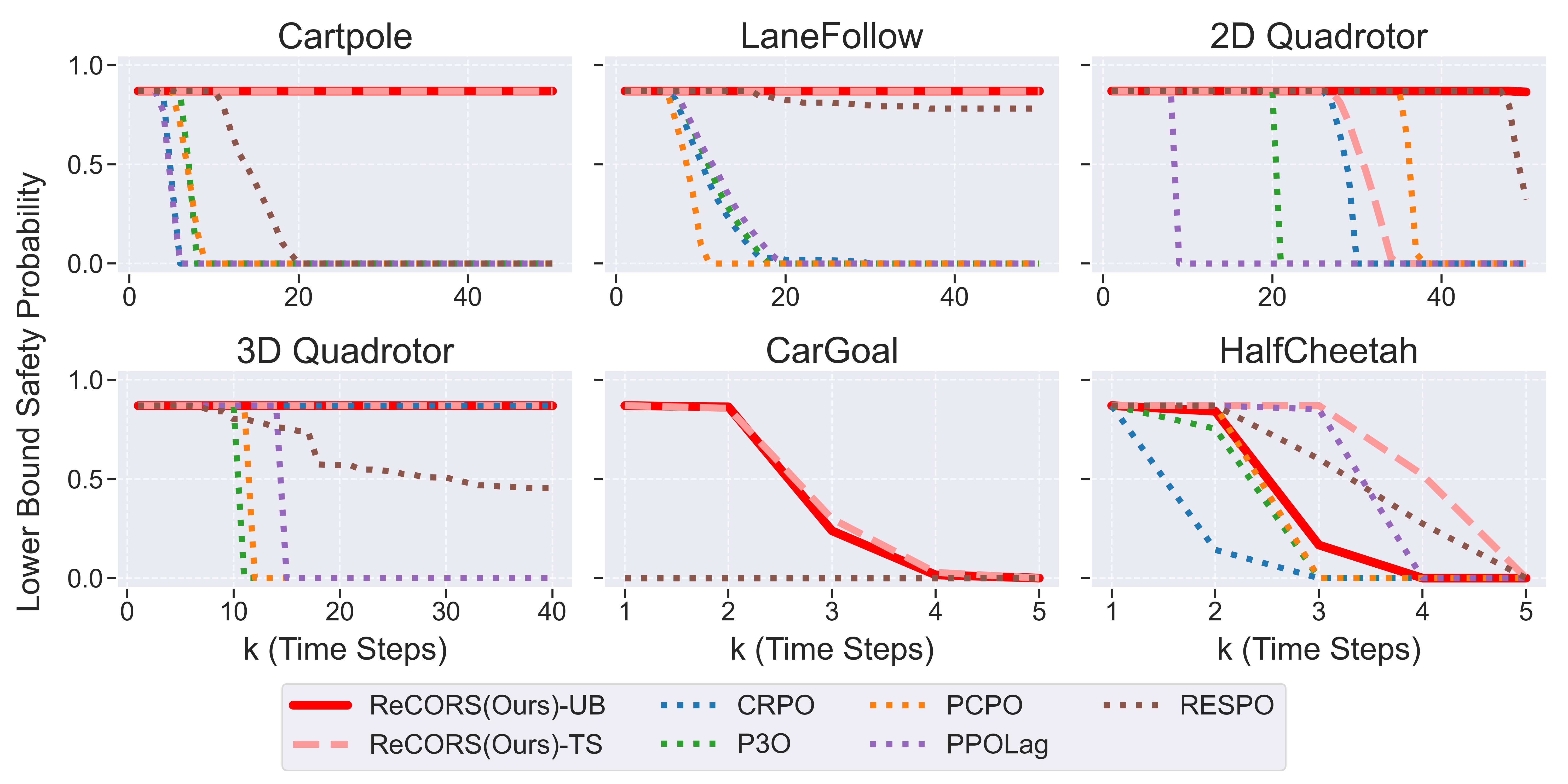}
    \caption{Probabilistic safety verification results. Higher is better.
    ReCORS-UB (red solid line) is applying conformal prediction per-step and union bounding the results.
    ReCORS-TS (pink dotted line) uses the time series conformal prediction method from~\cite{cleaveland2024conformal}.
    All the baseline plotted lines use the conformal prediction method that results in better verified safety results.
    }
    \Description{Probabilistic safety verification results. Higher is better.
    ReCORS-UB (red solid line) is applying conformal prediction per-step and union bounding the results.
    ReCORS-TS (pink dotted line) uses the time series conformal prediction method from~\cite{cleaveland2024conformal}.
    All the baseline plotted lines use the conformal prediction method that results in better verified safety results.}
    \label{fig:veri_results}
\end{figure*}

\begin{figure*}[h!]
    \centering
    \includegraphics[width=0.65\linewidth]{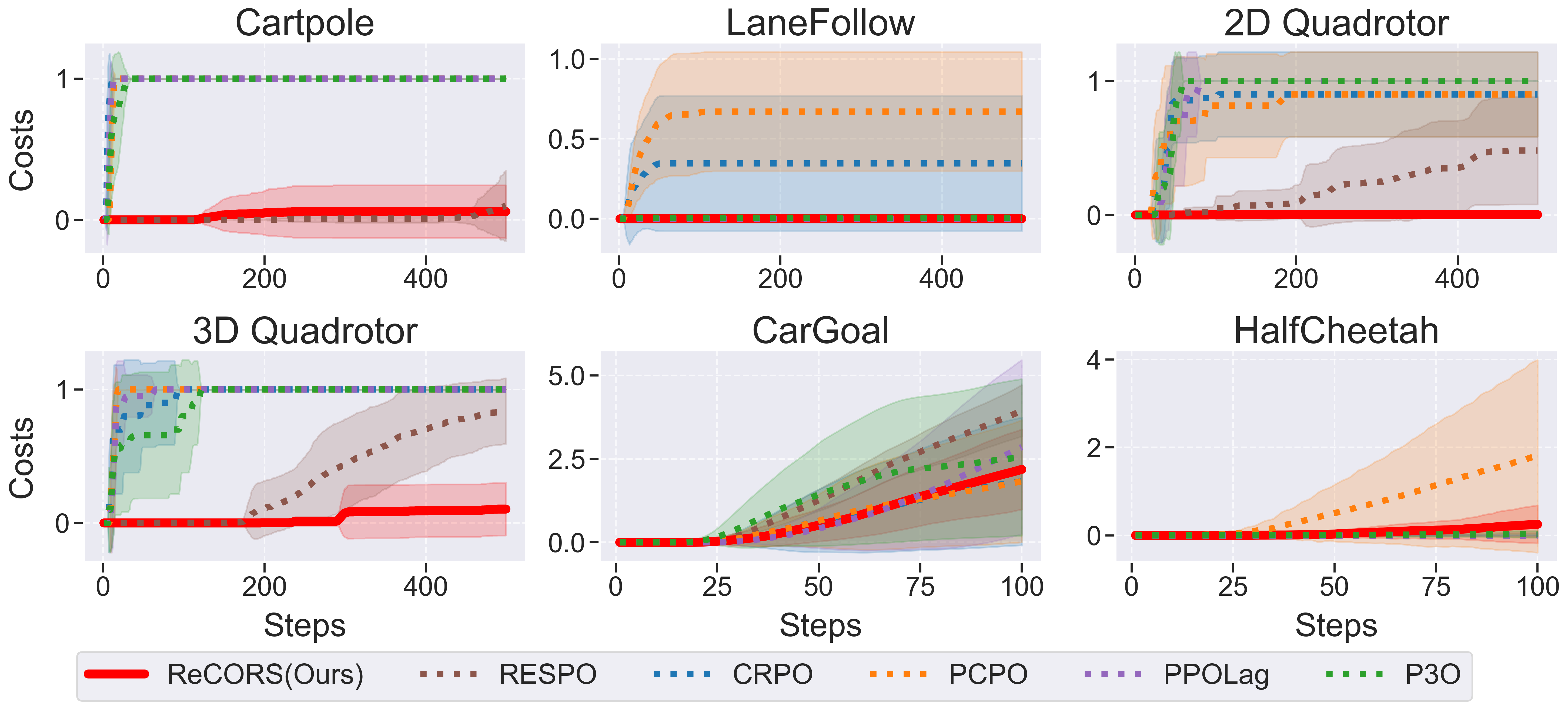}
    \caption{Empirical safety evaluation, comparing safety cost ReCORS (red line) with baselines. Lower is better.}
    \Description{Empirical safety evaluation, comparing safety cost ReCORS (red line) with baselines. Lower is better.}
    \label{fig:empirical_cost}
\end{figure*}

\begin{figure}[h!]
    \centering
    \includegraphics[width=0.95\linewidth]{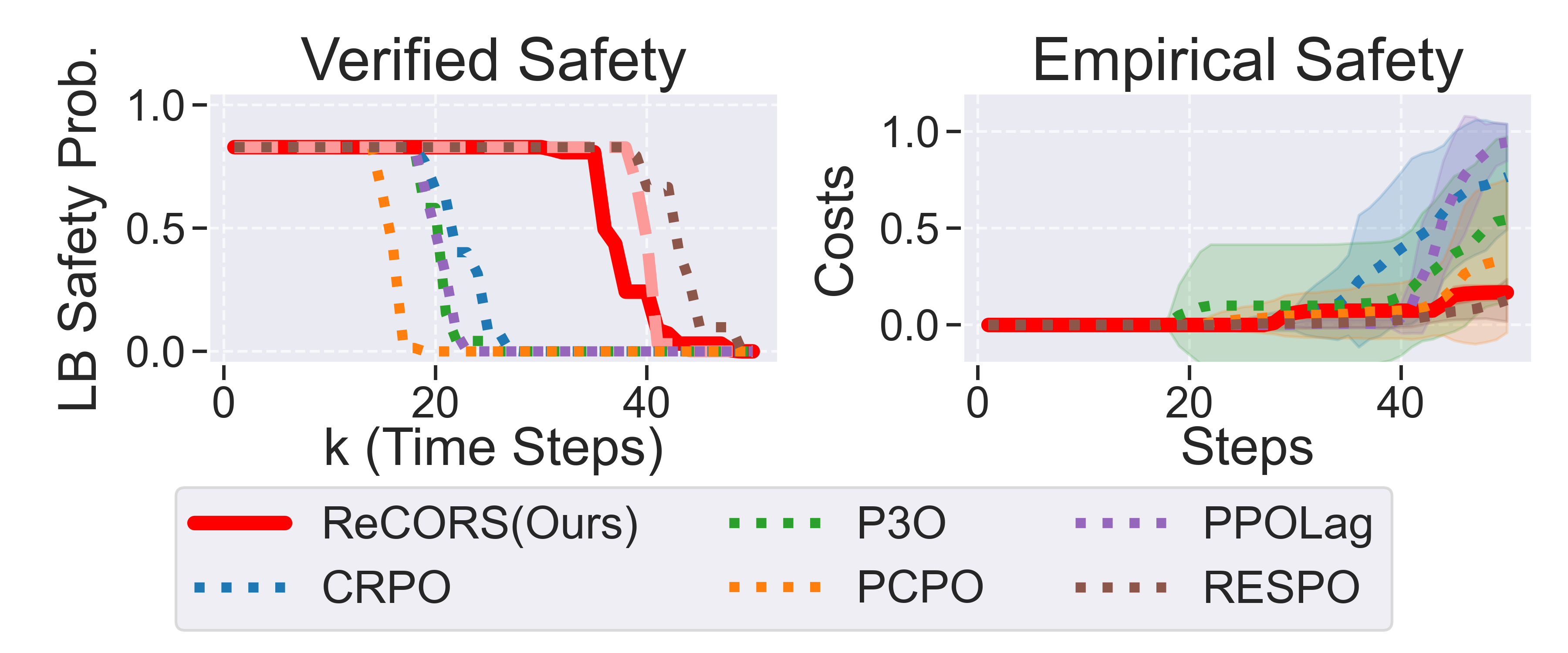}
    \caption{Verified safety (left) and empirical safety (right) for 2D Quadrotor with nonlinear safe distance constraint.}
    \Description{Verified safety (left) and empirical safety (right) for 2D Quadrotor environment with nonlinear safe distance constraint.}
    \label{fig:nonlinear_safety}
\end{figure}

\begin{figure*}[h!]
    \centering
    \includegraphics[width=0.75\linewidth]{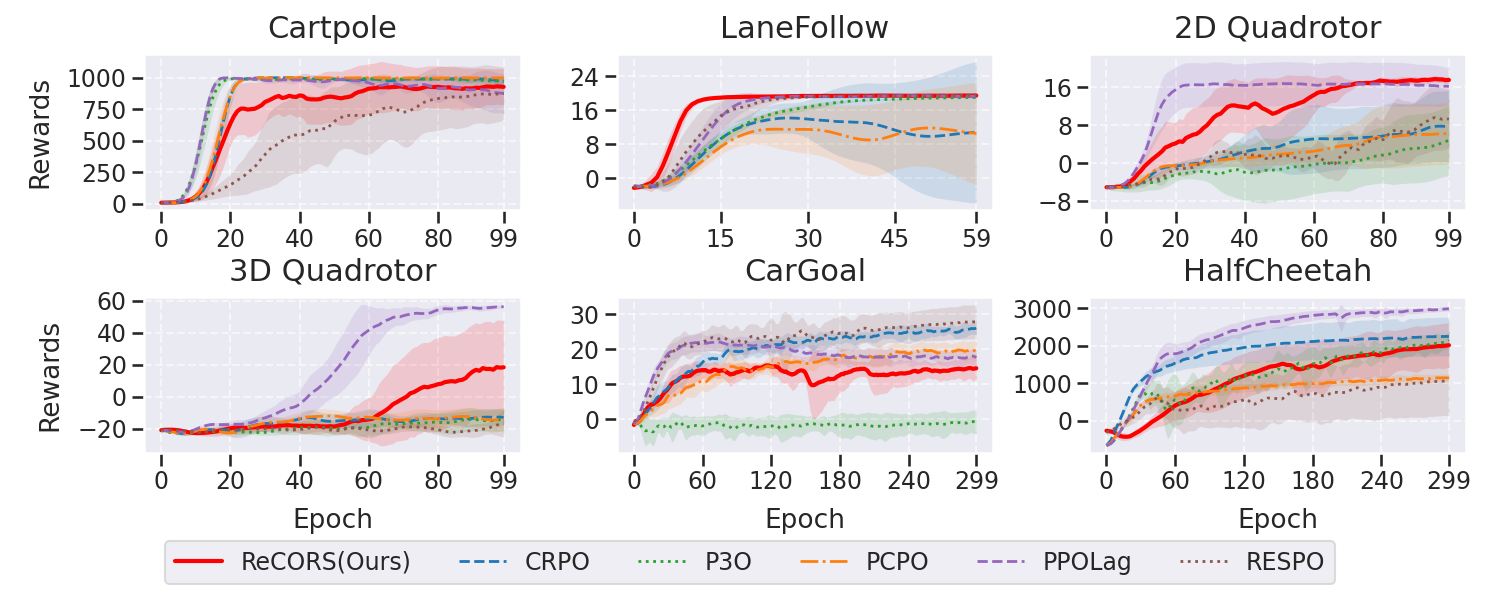}
    \caption{Reward comparisons of ReCORS (red line) with baselines. Higher is better.}
    \Description{Reward comparisons of ReCORS (red line) with baselines. Higher is better.}
    \label{fig:empirical_rewards}
\end{figure*}

\begin{figure}[h!]
    \centering
    \includegraphics[width=0.55\linewidth]{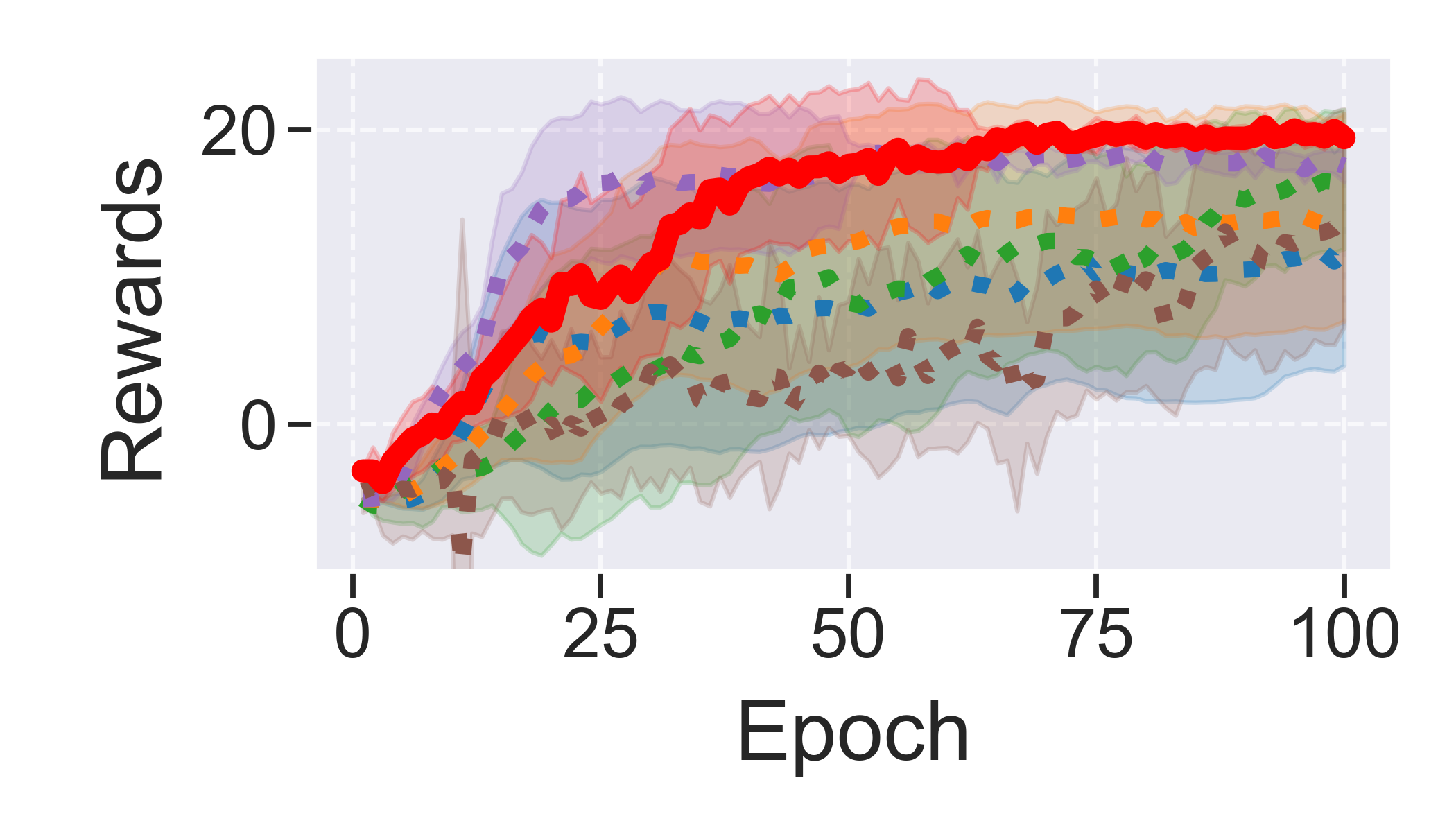}
    \caption{Reward comparisons in 2D Quadrotor environment with nonlinear safe distance constraint.}
    \Description{Reward comparisons in 2D Quadrotor environment with nonlinear safe distance constraint.}
    \label{fig:nonlinear_reward}
\end{figure}

\subsection{Results}

\subsubsection{Verified Safety Evaluation}


We begin by comparing ReCORS to the baselines in terms of verified lower bound on the safety probability ($y$-axis) as a function of the decision horizon $K$ ($x$-axis).
The results with affine safety constraints are provided in Figure \ref{fig:veri_results}, and exhibit notable improvement over the baselines in every environment.
For example, in CartPole, none of the baselines provide useful safety guarantees beyond $K=20$, while ReCORS guarantees safety with high probability through $K\ge 50$.
In 2D and 3D Quadrotor, RESPO and CRPO can be verified safe for relatively long horizons just as ReCORS, but this is a superficial consequence of the drone learning to essentially stay in place for both these baselines (obtaining essentially minimal reward in both cases; see Figure~\ref{fig:empirical_rewards}), while ReCORS achieves significantly higher reward than these \emph{without compromising safety guarantees}.
Curiously, while in most cases the tighter bound either coincides with or improves on the union bound, the latter yields the best performance in the 2D quadrotor case; indeed, in most cases the two bounds essentially coincide.
The reason stems from the fact that our trained systems often have bounds that are relatively similar over time, so that the optimal time-series bound parameters are often essentially uniform, corresponding to the union bound.

Figure~\ref{fig:nonlinear_safety} (left) compares verified safety of ReCORS against baselines in the 2D Quadrotor environment with a nonlinear safe-distance constraint.
ReCORS significantly outperforms all baselines except RESPO, but this is again a superficial consequence of RESPO learning resulting in the drone staying in place.
Consequently, ReCORS reward is significantly higher than RESPO (see Figure~\ref{fig:nonlinear_reward}).



\subsubsection{Empirical Safety Evaluation}

Next, we consider to what extent policies are \emph{empirically safe} over far longer horizons than we can verify.
We quantify this in terms of \emph{costs of safety violations}, which is just the fraction of time steps for which safety constraints are violated.
Figure~\ref{fig:empirical_cost} presents the results for affine safety, while nonlinear safety results are in Figure~\ref{fig:nonlinear_safety} (right).
The results are roughly consistent with verified safety, showing that ReCORS results in exceedingly few safety violations over long decision horizons, and is in all cases either far better than the baselines (e.g., 3D quadrotor), or comparable to the best baseline.
In the latter instances, the baselines which are competitive fair far worse in terms of reward.


\subsubsection{Empirical Performance Evaluation}

Finally, we compare performance of ReCORS and the baselines in Figure \ref{fig:empirical_rewards}, which shows the mean and standard deviation of empirical rewards across $10$ random seeds as a function of training epochs.
Since we expect a natural tradeoff between safety---particularly, verified safety---and reward, it is no surprise that the relative performance of ReCORS varies, and in several domains other methods, such as PPOLag and CRPO, exhibit better performance (reward).
However, in all cases in which other methods achieve better reward than ReCORS, they exhibit very poor safety properties; indeed,
ReCORS \textit{is the only method that can achieve both strong performance and verified safety.}
To illustrate, consider the HalfCheetah environment, where CRPO and PPOLag achieve higher reward.
Here, CRPO achieves the worst safety, and while PPOLag is better, ReCORS achieves significantly better verified safety.
In several domains, such as 2D Quadrotor, ReCORS achieves \emph{both}, superior safety and reward than all baselines.

\begin{figure}[h!]
    \centering
    \includegraphics[width=0.95\linewidth]{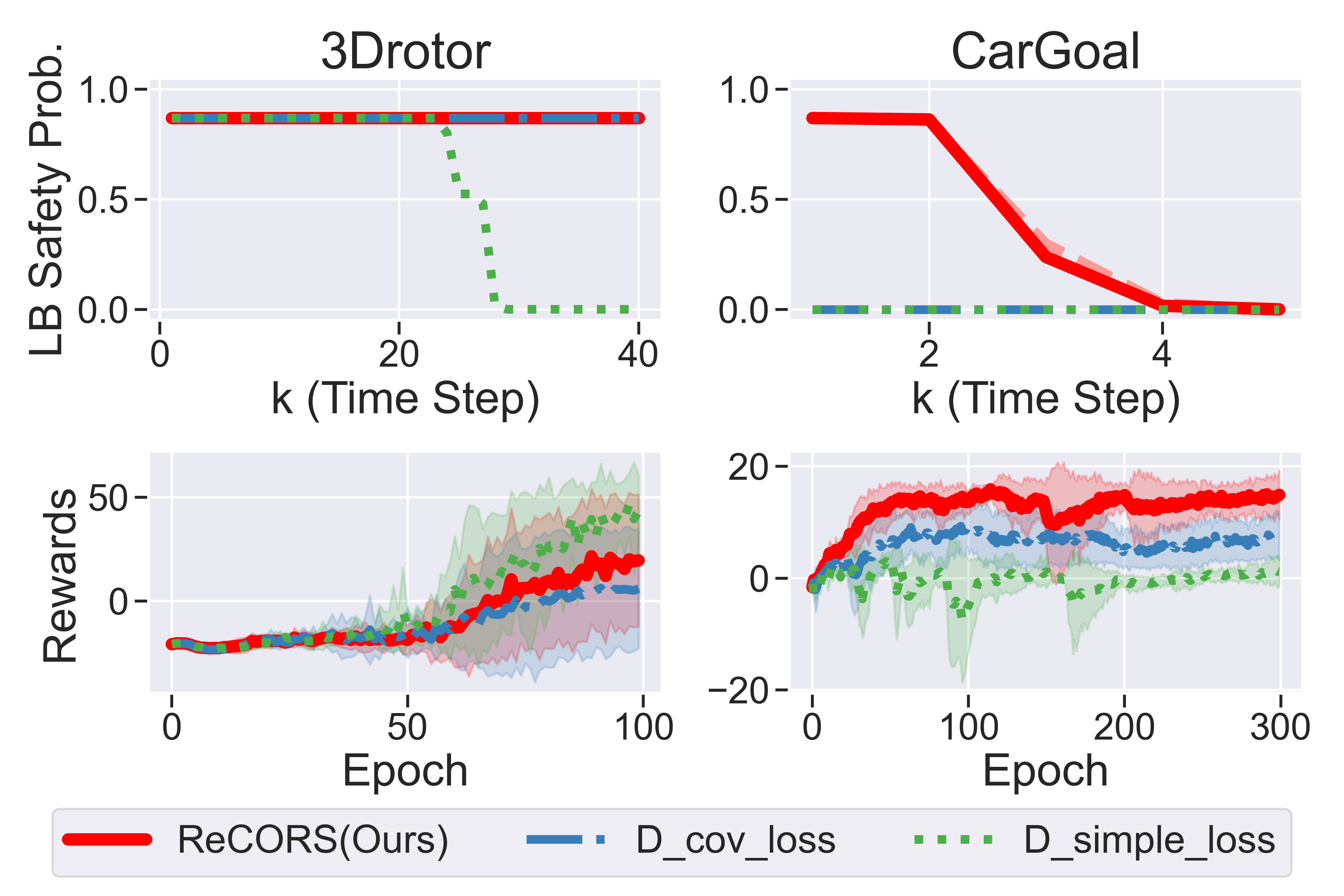}
    \caption{Ablation results for 3D Quadrotor (left) and CarGoal (right). D\_cov\_loss jointly optimizes $\beta$ and $\omega$ in the conformal coverage loss, while ReCORS only uses the latter ($\omega$). D\_simple\_loss uses uniform weights in learning $\hat{f}_\beta$.
    }
    \Description{Ablation results for 3D Quadrotor (left) and CarGoal (right). D\_cov\_loss jointly optimizes $\beta$ and $\omega$ in the conformal coverage loss, while ReCORS only uses the latter ($\omega$). D\_simple\_loss uses uniform weights in learning $\hat{f}_\beta$.}
    \label{fig:ablation_plot}
\end{figure}

\subsection{Ablations}
\label{S:exp_ablations}

Finally, we provide ablation results on two design choices: 1) including a heuristic safety weight in training the dynamics model, and 2) jointly optimizing the dynamics and uncertainty model in the conformal coverage loss $\mathcal{L}_{\text{cov}}(\beta,\omega)$.
Results are shown in Figure~\ref{fig:ablation_plot}.
We find that we have the best performance omitting dependence on $\beta$ in the coverage loss (ReCORS is better than D\_conv\_loss).
Additionally, using safety-aware weights for learning $\hat{f}_\beta$ provides better verifies safety while maintaining good performances (ReCORS is safer than D\_simple\_loss).

\section{Conclusion}\label{sec:conclusion}

We present a novel approach for learning effective policies in general dynamical systems while abiding by probabilistic safety constraints.
Our approach is designed to deal with unknown stochastic dynamics, and avoids restrictive assumptions on either true dynamics or how they are learned.
This is accomplished by combining conformal predictions, finite-horizon reachability proofs, and finite-sample tail bounds to obtain verified safety, with model-free RL training methods facilitating strong performance.
While we achieve state-of-the-art results on several standard benchmarks, many open problems remain, paramount among these being an assumption of perfect observability of the system state.

\section*{Acknowledgements}
This work was partially supported by the NSF (IIS-2214141, CCF-2403758), ARO (W911NF-25-1-0059), ONR (N000142412663), Foresight Institute, and Amazon.



\balance
\bibliographystyle{ACM-Reference-Format} 
\bibliography{ref}

@article{bai2022efficient,
  title={Efficient and differentiable conformal prediction with general function classes},
  author={Bai, Yu and Mei, Song and Wang, Huan and Zhou, Yingbo and Xiong, Caiming},
  journal={arXiv preprint arXiv:2202.11091},
  year={2022}
}

@inproceedings{stooke2020responsive,
  title = {Responsive Safety in Reinforcement Learning by Pid Lagrangian Methods},
  booktitle = {International Conference on Machine Learning},
  author = {Stooke, Adam and Achiam, Joshua and Abbeel, Pieter},
  year = {2020},
  pages = {9133--9143},
  publisher = {PMLR}
}

@inproceedings{ma2022conservative,
  title = {Conservative and Adaptive Penalty for Model-Based Safe Reinforcement Learning},
  booktitle = {Proceedings of the {{AAAI}} Conference on Artificial Intelligence},
  author = {Ma, Yecheng Jason and Shen, Andrew and Bastani, Osbert and Dinesh, Jayaraman},
  year = {2022},
  volume = {36},
  pages = {5404--5412}
}

@inproceedings{so2023solving,
  title = {Solving Stabilize-Avoid Optimal Control via Epigraph Form and Deep Reinforcement Learning},
  booktitle = {Proceedings of Robotics: {{Science}} and Systems},
  author = {So, Oswin and Fan, Chuchu},
  year = {2023}
}

@article{jayant2022model,
  title = {Model-Based Safe Deep Reinforcement Learning via a Constrained Proximal Policy Optimization Algorithm},
  author = {Jayant, Ashish K and Bhatnagar, Shalabh},
  year = {2022},
  journal = {Advances in Neural Information Processing Systems},
  volume = {35},
  pages = {24432--24445}
}

@article{zhou2022neural,
  title={Neural Lyapunov control of unknown nonlinear systems with stability guarantees},
  author={Zhou, Ruikun and Quartz, Thanin and De Sterck, Hans and Liu, Jun},
  journal={Advances in Neural Information Processing Systems},
  volume={35},
  pages={29113--29125},
  year={2022}
}

@article{wang2024providing,
  title={Providing safety assurances for systems with unknown dynamics},
  author={Wang, Hao and Borquez, Javier and Bansal, Somil},
  journal={IEEE Control Systems Letters},
  year={2024},
}

@book{vovk2005algorithmic,
  title={Algorithmic learning in a random world},
  author={Vovk, Vladimir and Gammerman, Alexander and Shafer, Glenn},
  volume={29},
  year={2005},
  publisher={Springer}
}

@article{angelopoulos2021gentle,
  title={A gentle introduction to conformal prediction and distribution-free uncertainty quantification},
  author={Angelopoulos, Anastasios N and Bates, Stephen},
  journal={arXiv preprint arXiv:2107.07511},
  year={2021}
}

@article{zhang2021reinforcement,
  title={Reinforcement learning for robot research: A comprehensive review and open issues},
  author={Zhang, Tengteng and Mo, Hongwei},
  journal={International Journal of Advanced Robotic Systems},
  volume={18},
  number={3},
  pages={1-21},
  year={2021},
}

@article{gu2022review,
  title={A review of safe reinforcement learning: Methods, theory and applications},
  author={Gu, Shangding and Yang, Long and Du, Yali and Chen, Guang and Walter, Florian and Wang, Jun and Knoll, Alois},
  journal={arXiv preprint arXiv:2205.10330},
  year={2022}
}

@inproceedings{wu2024verified,
  title={Verified Safe Reinforcement Learning for Neural Network Dynamic Models},
  author={Wu, Junlin and Zhang, Huan and Vorobeychik, Yevgeniy},
  booktitle={Neural Information Processing Systems},
  year={2024}
}

@inproceedings{salamati2022data,
  title={Data-driven safety verification of stochastic systems via barrier certificates: A wait-and-judge approach},
  author={Salamati, Ali and Zamani, Majid},
  booktitle={Learning for Dynamics and Control Conference},
  pages={441--452},
  year={2022},
  organization={PMLR}
}

@phdthesis{fan2019formal,
  title={Formal methods for safe autonomy: Data-driven verification, synthesis, and applications},
  author={Fan, Chuchu},
  year={2019},
  school={University of Illinois at Urbana-Champaign}
}

@incollection{fan2019data,
  title={Data-driven safety verification of complex cyber-physical systems},
  author={Fan, Chuchu and Mitra, Sayan},
  booktitle={Design Automation of Cyber-Physical Systems},
  pages={107--142},
  year={2019},
  publisher={Springer}
}

@inproceedings{fan2017dryvr,
  title={DryVR: Data-driven verification and compositional reasoning for automotive systems},
  author={Fan, Chuchu and Qi, Bolun and Mitra, Sayan and Viswanathan, Mahesh},
  booktitle={International Conference on Computer Aided Verification},
  pages={441--461},
  year={2017},
  organization={Springer}
}

@inproceedings{yu2022reachability,
  title={Reachability constrained reinforcement learning},
  author={Yu, Dongjie and Ma, Haitong and Li, Shengbo and Chen, Jianyu},
  booktitle={International Conference on Machine Learning},
  pages={25636--25655},
  year={2022},
}

@article{ganai2023iterative,
  title={Iterative reachability estimation for safe reinforcement learning},
  author={Ganai, Milan and Gong, Zheng and Yu, Chenning and Herbert, Sylvia and Gao, Sicun},
  journal={Neural Information Processing Systems},
  volume={36},
  pages={69764--69797},
  year={2023}
}

@article{ma2025learning,
  title={Learning Vision-Based Neural Network Controllers with Semi-Probabilistic Safety Guarantees},
  author={Ma, Xinhang and Wu, Junlin and Sibai, Hussein and Kantaros, Yiannis and Vorobeychik, Yevgeniy},
  journal={arXiv preprint arXiv:2503.00191},
  year={2025}
}

@article{yang2020projection,
  title={Projection-based constrained policy optimization},
  author={Yang, Tsung-Yen and Rosca, Justinian and Narasimhan, Karthik and Ramadge, Peter J},
  journal={arXiv preprint arXiv:2010.03152},
  year={2020}
}

@inproceedings{xu2021crpo,
  title={{CRPO}: A new approach for safe reinforcement learning with convergence guarantee},
  author={Xu, Tengyu and Liang, Yingbin and Lan, Guanghui},
  booktitle={International Conference on Machine Learning},
  pages={11480--11491},
  year={2021},
}

@article{zhang2022penalized,
  title={Penalized proximal policy optimization for safe reinforcement learning},
  author={Zhang, Linrui and Shen, Li and Yang, Long and Chen, Shixiang and Yuan, Bo and Wang, Xueqian and Tao, Dacheng},
  journal={arXiv preprint arXiv:2205.11814},
  year={2022}
}

@article{ray2019benchmarking,
  title={Benchmarking safe exploration in deep reinforcement learning},
  author={Ray, Alex and Achiam, Joshua and Amodei, Dario},
  journal={arXiv preprint arXiv:1910.01708},
  volume={7},
  number={1},
  pages={2},
  year={2019}
}

@article{ji2024omnisafe,
  title={Omnisafe: An infrastructure for accelerating safe reinforcement learning research},
  author={Ji, Jiaming and Zhou, Jiayi and Zhang, Borong and Dai, Juntao and Pan, Xuehai and Sun, Ruiyang and Huang, Weidong and Geng, Yiran and Liu, Mickel and Yang, Yaodong},
  journal={Journal of Machine Learning Research},
  volume={25},
  number={285},
  pages={1--6},
  year={2024}
}

@article{xu2021fast,
  title={Fast and complete: Enabling complete neural network verification with rapid and massively parallel incomplete verifiers},
  author={Xu, Kaidi and Zhang, Huan and Wang, Shiqi and Wang, Yihan and Jana, Suman and Lin, Xue and Hsieh, Cho-Jui},
  journal={arXiv preprint arXiv:2011.13824},
  year={2020}
}

@inproceedings{ji2023safety,
  title={Safety Gymnasium: A Unified Safe Reinforcement Learning Benchmark},
  author={Jiaming Ji and Borong Zhang and Jiayi Zhou and Xuehai Pan and Weidong Huang and Ruiyang Sun and Yiran Geng and Yifan Zhong and Josef Dai and Yaodong Yang},
  booktitle={Thirty-seventh Conference on Neural Information Processing Systems Datasets and Benchmarks Track},
  year={2023},
  url={https://openreview.net/forum?id=WZmlxIuIGR}
}

@book{altman2021constrained,
  title={Constrained Markov Decision Processes},
  author={Altman, Eitan},
  year={2021},
  publisher={Routledge}
}

@article{tessler2018reward,
  title={Reward constrained policy optimization},
  author={Tessler, Chen and Mankowitz, Daniel J and Mannor, Shie},
  journal={arXiv preprint arXiv:1805.11074},
  year={2018}
}

@inproceedings{achiam2017constrained,
  title={Constrained policy optimization},
  author={Achiam, Joshua and Held, David and Tamar, Aviv and Abbeel, Pieter},
  booktitle={International conference on machine learning},
  pages={22--31},
  year={2017},
}

@article{ma2021feasible,
  title={Feasible actor-critic: Constrained reinforcement learning for ensuring statewise safety},
  author={Ma, Haitong and Guan, Yang and Li, Shegnbo Eben and Zhang, Xiangteng and Zheng, Sifa and Chen, Jianyu},
  journal={arXiv preprint arXiv:2105.10682},
  year={2021}
}

@inproceedings{dawson2022safe,
  title={Safe nonlinear control using robust neural lyapunov-barrier functions},
  author={Dawson, Charles and Qin, Zengyi and Gao, Sicun and Fan, Chuchu},
  booktitle={Conference on Robot Learning},
  pages={1724--1735},
  year={2022},
  organization={PMLR}
}

@inproceedings{cheng2019end,
  title={End-to-end safe reinforcement learning through barrier functions for safety-critical continuous control tasks},
  author={Cheng, Richard and Orosz, G{\'a}bor and Murray, Richard M and Burdick, Joel W},
  booktitle={Proceedings of the AAAI conference on artificial intelligence},
  volume={33},
  pages={3387--3395},
  year={2019}
}

@inproceedings{wang2023enforcing,
  title={Enforcing hard constraints with soft barriers: Safe reinforcement learning in unknown stochastic environments},
  author={Wang, Yixuan and Zhan, Simon Sinong and Jiao, Ruochen and Wang, Zhilu and Jin, Wanxin and Yang, Zhuoran and Wang, Zhaoran and Huang, Chao and Zhu, Qi},
  booktitle={International Conference on Machine Learning},
  pages={36593--36604},
  year={2023},
  organization={PMLR}
}

@article{schulman2017proximal,
  title={Proximal policy optimization algorithms},
  author={Schulman, John and Wolski, Filip and Dhariwal, Prafulla and Radford, Alec and Klimov, Oleg},
  journal={arXiv preprint arXiv:1707.06347},
  year={2017}
}

@inproceedings{fisac2019bridging,
  title={Bridging hamilton-jacobi safety analysis and reinforcement learning},
  author={Fisac, Jaime F and Lugovoy, Neil F and Rubies-Royo, Vicen{\c{c}} and Ghosh, Shromona and Tomlin, Claire J},
  booktitle={2019 International Conference on Robotics and Automation (ICRA)},
  pages={8550--8556},
  year={2019},
  organization={IEEE}
}

@article{kochdumper2023provably,
  title={Provably safe reinforcement learning via action projection using reachability analysis and polynomial zonotopes},
  author={Kochdumper, Niklas and Krasowski, Hanna and Wang, Xiao and Bak, Stanley and Althoff, Matthias},
  journal={IEEE Open Journal of Control Systems},
  volume={2},
  pages={79--92},
  year={2023},
  publisher={IEEE}
}

@article{selim2022safe,
  title={Safe reinforcement learning using black-box reachability analysis},
  author={Selim, Mahmoud and Alanwar, Amr and Kousik, Shreyas and Gao, Grace and Pavone, Marco and Johansson, Karl H},
  journal={IEEE Robotics and Automation Letters},
  volume={7},
  number={4},
  pages={10665--10672},
  year={2022},
  publisher={IEEE}
}

@article{yu2023safe,
  title={Safe model-based reinforcement learning with an uncertainty-aware reachability certificate},
  author={Yu, Dongjie and Zou, Wenjun and Yang, Yujie and Ma, Haitong and Li, Shengbo Eben and Yin, Yuming and Chen, Jianyu and Duan, Jingliang},
  journal={IEEE Transactions on Automation Science and Engineering},
  volume={21},
  number={3},
  pages={4129--4142},
  year={2023},
  publisher={IEEE}
}

@article{azar2021drone,
  title={Drone deep reinforcement learning: A review},
  author={Azar, Ahmad Taher and Koubaa, Anis and Ali Mohamed, Nada and Ibrahim, Habiba A and Ibrahim, Zahra Fathy and Kazim, Muhammad and Ammar, Adel and Benjdira, Bilel and Khamis, Alaa M and Hameed, Ibrahim A and others},
  journal={Electronics},
  volume={10},
  number={9},
  pages={999},
  year={2021},
}

@article{xu2024omnidrones,
  title={Omnidrones: An efficient and flexible platform for reinforcement learning in drone control},
  author={Xu, Botian and Gao, Feng and Yu, Chao and Zhang, Ruize and Wu, Yi and Wang, Yu},
  journal={IEEE Robotics and Automation Letters},
  volume={9},
  number={3},
  pages={2838--2844},
  year={2024},
}

@article{kiran2021deep,
  title={Deep reinforcement learning for autonomous driving: A survey},
  author={Kiran, B Ravi and Sobh, Ibrahim and Talpaert, Victor and Mannion, Patrick and Al Sallab, Ahmad A and Yogamani, Senthil and P{\'e}rez, Patrick},
  journal={IEEE transactions on intelligent transportation systems},
  volume={23},
  number={6},
  pages={4909--4926},
  year={2021},
}

@article{han2023survey,
  title={A survey on deep reinforcement learning algorithms for robotic manipulation},
  author={Han, Dong and Mulyana, Beni and Stankovic, Vladimir and Cheng, Samuel},
  journal={Sensors},
  volume={23},
  number={7},
  pages={3762},
  year={2023},
}

@inproceedings{ames2019control,
  title={Control barrier functions: Theory and applications},
  author={Ames, Aaron D and Coogan, Samuel and Egerstedt, Magnus and Notomista, Gennaro and Sreenath, Koushil and Tabuada, Paulo},
  booktitle={European Control Conference},
  pages={3420--3431},
  year={2019},
}

@article{ames2016control,
  title={Control barrier function based quadratic programs for safety critical systems},
  author={Ames, Aaron D and Xu, Xiangru and Grizzle, Jessy W and Tabuada, Paulo},
  journal={IEEE Transactions on Automatic Control},
  volume={62},
  number={8},
  pages={3861--3876},
  year={2016},
}

@inproceedings{ivanov2019verisig,
  title={Verisig: verifying safety properties of hybrid systems with neural network controllers},
  author={Ivanov, Radoslav and Weimer, James and Alur, Rajeev and Pappas, George J and Lee, Insup},
  booktitle={ACM International Conference on Hybrid Systems: Computation and Control},
  pages={169--178},
  year={2019}
}

@inproceedings{zhang2024seev,
  title={SEEV: Synthesis with Efficient Exact Verification for ReLU Neural Barrier Functions},
  author={Zhang, Hongchao and Qin, Zhizhen and Gao, Sicun and Clark, Andrew},
  booktitle={Neural Information Processing Systems},
  pages={101367--101392},
  year={2024}
}

@inproceedings{zhang2023exact,
  title={Exact verification of relu neural control barrier functions},
  author={Zhang, Hongchao and Wu, Junlin and Vorobeychik, Yevgeniy and Clark, Andrew},
  booktitle={Neural Information Processing Systems},
  pages={5685--5705},
  year={2023}
}

@article{tran2019safety,
  title={Safety verification of cyber-physical systems with reinforcement learning control},
  author={Tran, Hoang-Dung and Cai, Feiyang and Diego, Manzanas Lopez and Musau, Patrick and Johnson, Taylor T and Koutsoukos, Xenofon},
  journal={ACM Transactions on Embedded Computing Systems},
  volume={18},
  number={5s},
  pages={1--22},
  year={2019},
}

@inproceedings{jackson2020safety,
  title={Safety verification of unknown dynamical systems via gaussian process regression},
  author={Jackson, John and Laurenti, Luca and Frew, Eric and Lahijanian, Morteza},
  booktitle={IEEE Conference on Decision and Control},
  pages={860--866},
  year={2020},
}

@article{chow2018lyapunov,
  title={A lyapunov-based approach to safe reinforcement learning},
  author={Chow, Yinlam and Nachum, Ofir and Duenez-Guzman, Edgar and Ghavamzadeh, Mohammad},
  journal={Advances in neural information processing systems},
  volume={31},
  year={2018}
}

@inproceedings{kong2015kinematic,
  title={Kinematic and dynamic vehicle models for autonomous driving control design},
  author={Kong, Jason and Pfeiffer, Mark and Schildbach, Georg and Borrelli, Francesco},
  booktitle={2015 IEEE intelligent vehicles symposium (IV)},
  pages={1094--1099},
  year={2015},
  organization={IEEE}
}

@article{barto1983neuronlike,
  title={Neuronlike adaptive elements that can solve difficult learning control problems},
  author={Barto, Andrew G and Sutton, Richard S and Anderson, Charles W},
  journal={IEEE transactions on systems, man, and cybernetics},
  pages={834--846},
  year={1983},
  publisher={IEEE}
}

@article{emam2021safe,
  title={Safe model-based reinforcement learning using robust control barrier functions},
  author={Emam, Yousef and Glotfelter, Paul and Kira, Zsolt and Egerstedt, Magnus},
  journal={arXiv preprint arXiv:2110.05415},
  year={2021}
}

@inproceedings{zhao2023probabilistic,
  title={Probabilistic safeguard for reinforcement learning using safety index guided gaussian process models},
  author={Zhao, Weiye and He, Tairan and Liu, Changliu},
  booktitle={Learning for Dynamics and Control Conference},
  pages={783--796},
  year={2023},
  organization={PMLR}
}

@article{dai2021lyapunov,
  title={Lyapunov-stable neural-network control},
  author={Dai, Hongkai and Landry, Benoit and Yang, Lujie and Pavone, Marco and Tedrake, Russ},
  journal={arXiv preprint arXiv:2109.14152},
  year={2021}
}

@inproceedings{cleaveland2024conformal,
  title={Conformal prediction regions for time series using linear complementarity programming},
  author={Cleaveland, Matthew and Lee, Insup and Pappas, George J and Lindemann, Lars},
  booktitle={Proceedings of the AAAI Conference on Artificial Intelligence},
  volume={38},
  number={19},
  pages={20984--20992},
  year={2024}
}

@inproceedings{shi2025neural,
  title={Neural network verification with branch-and-bound for general nonlinearities},
  author={Shi, Zhouxing and Jin, Qirui and Kolter, Zico and Jana, Suman and Hsieh, Cho-Jui and Zhang, Huan},
  booktitle={International Conference on Tools and Algorithms for the Construction and Analysis of Systems},
  pages={315--335},
  year={2025},
  organization={Springer}
}

@inproceedings{hashemi2023data,
  title={Data-driven reachability analysis of stochastic dynamical systems with conformal inference},
  author={Hashemi, Navid and Qin, Xin and Lindemann, Lars and Deshmukh, Jyotirmoy V},
  booktitle={IEEE Conference on Decision and Control},
  pages={3102--3109},
  year={2023},
}

@inproceedings{muthali2023multi,
  title={Multi-agent reachability calibration with conformal prediction},
  author={Muthali, Anish and Shen, Haotian and Deglurkar, Sampada and Lim, Michael H and Roelofs, Rebecca and Faust, Aleksandra and Tomlin, Claire},
  booktitle={2023 62nd IEEE Conference on Decision and Control (CDC)},
  pages={6596--6603},
  year={2023},
}

@article{lindemann2024formal,
  title={Formal verification and control with conformal prediction},
  author={Lindemann, Lars and Zhao, Yiqi and Yu, Xinyi and Pappas, George J and Deshmukh, Jyotirmoy V},
  journal={arXiv preprint arXiv:2409.00536},
  year={2024}
}

@article{hashemi2024statistical,
  title={Statistical reachability analysis of stochastic cyber-physical systems under distribution shift},
  author={Hashemi, Navid and Lindemann, Lars and Deshmukh, Jyotirmoy V},
  journal={IEEE Transactions on Computer-Aided Design of Integrated Circuits and Systems},
  volume={43},
  number={11},
  pages={4250--4261},
  year={2024},
}

@article{chakraborty2025safety,
  title={Safety Evaluation of Motion Plans Using Trajectory Predictors as Forward Reachable Set Estimators},
  author={Chakraborty, Kaustav and Feng, Zeyuan and Veer, Sushant and Sharma, Apoorva and Ding, Wenhao and Topan, Sever and Ivanovic, Boris and Pavone, Marco and Bansal, Somil},
  journal={arXiv preprint arXiv:2507.22389},
  year={2025}
}

\appendix
\clearpage

\section{Experiment Details}\label{sec:appendix_experiment_details}

\subsection{Benchmarks:}

\paragraph{Cartpole:}
For our Cartpole environment,
we implement the standard dynamics for a cartpole system~\cite{barto1983neuronlike}:
\begin{align*}
    \ddot{\theta} &= \frac{g\sin\theta - \cos\theta\left[\frac{F + ml\dot{\theta}^2\sin\theta}{m_c + m}\right]}{l\left[\frac{4}{3} - \frac{m\cos^2\theta}{m_c + m}\right]} \\
    \ddot{x} &= \frac{F + ml\dot{\theta}^2\sin\theta}{m_c + m} - \frac{ml\ddot{\theta}\cos\theta}{m_c + m}
\end{align*}
where
$x$ is the cart position,
$\theta$ is the pole angle from vertical,
$F$ is the force applied to the cart,
$g=9.8$m/s$^2$ is the gravitational acceleration,
$m_c=1$kg is the cart mass,
$m=0.1$kg is the pole mass,
$l=0.5$m is the half-length of the pole.
The system is discretized using Euler integration with timestep $\Delta t=0.05$ seconds.

The system state is $4$-dimensional: $(x,\dot{x},\theta,\dot{\theta})$ representing cart position, cart velocity, pole angle, and angular velocity.
The control objective is to keep the pendulum upright and maintain cart position. Safety constraints are:
\begin{enumerate}
    \item pendulum angle constraint $|\theta| \leq \theta_{\max} = 0.2$ radians
    \item cart position constraint $|x| \leq x_{\max} = 2.4$m
\end{enumerate}
Initial state region is uniformly sampled from $[-0.05, 0.05]$ for all state variables.
Reward function provides a constant reward of $1$ per timestep when the pole is upright (within constraint).
Cost function is a binary signal, 
with a cost of $1$ when any safety constraint is violated and $0$ otherwise.
The episode terminates immediately upon constraint violation. 

\paragraph{Lane Follow:}
For our lane following environment,
we implement the discrete-bicycle model for vehicle dynamics~\cite{kong2015kinematic}: 
\begin{align*}
    \dot{x} &= v \cos(\theta + \beta) \\
    \dot{y} &= v \sin(\theta + \beta) \\
    \dot{\theta} &= \frac{v}{l_r}\sin(\beta) \\
    \dot{v} &= a \\
    \beta &= \tan^{-1}\left(\frac{l_r}{l_f + l_r}\tan(\delta_f)\right)
\end{align*}
where
$x$ denotes lateral distance from the lane center,
$y$ denotes longitudinal position,
$\theta$ denotes heading angle relative to the lane,
$v$ is the vehicle speed,
$a$ is the acceleration,
$\delta_f$ is the front wheel steering angle,
$\beta$ is the side slip angle,
and $l_f$ and $l_r$ are distances from the center of gravity to front and rear axles respectively. 
Timestep $\Delta t$ is set to $0.05$ seconds.

We set the wheelbase to be $2.9$m. The system state is $3$-dimensional: $(x,\theta,v)$ representing lateral deviation from lane center, heading error, and vehicle speed. 
The objective is to follow the lane centerline. 
Safety constraints are:
\begin{enumerate}
    \item lateral distance constraint $|x| \leq d_{\max}=0.7$m
    \item heading error constraint $|\theta| \leq \theta_{\max} = \pi/4$ radian
\end{enumerate}
The initial state region is $x \in [-0.35, 0.35]$, $\theta \in [-\pi/8, \pi/8]$, and $v \in [2.4, 3.6]$.
Reward function encourages the vehicle to stay close to the center with minimal heading error, specifically:
\(
r = -(x^2 + 0.5 \cdot (\theta/\theta_{\max})^2)+0.1
\).
Cost function is a binary safety signal, with a cost of $1$ when any safety constraint is violated and $0$ otherwise. 
The episode terminates immediately after a constraint violation.

\paragraph{2D Quadrotor:}

For our 2D Quadrotor environment, we follow the settings in~\cite{emam2021safe} for the drone dynamics:
\begin{align*}
    \ddot{x} &= -\frac{(u_1 + u_2)\sin\theta}{m} \\
    \ddot{y} &= \frac{(u_1 + u_2)\cos\theta}{m} - g \\
    \ddot{\theta} &= \frac{(u_1 - u_2)l}{I}
\end{align*}
where
$x$ and $y$ are horizontal and vertical positions,
$\theta$ is the pitch angle,
$u_1$ and $u_2$ are the forces from the two rotors,
$m=1.0$kg is the vehicle mass,
$I=0.01$kg$\cdot$m$^2$ is moment of inertia,
$l=0.25$m is distance from center of mass to rotor,
$g=9.81$m/s$^2$ is gravitational acceleration.
System timestep $\Delta t$ is set to $0.02$ seconds.

The agent has $12$-dimensional observation space
with a $6$-dimensional state space: $(x,y,\theta,\dot{x},\dot{y},\dot{\theta})$ representing position, orientation, and their velocities. 
The other $6$ dimensions in the observation space represents the goal compass direction and distance and closest hazard compass direction and distance.
The objective is to navigate to a target location while avoiding static obstacles.
Safety constraints are:
\begin{enumerate}
    \item position boundaries constraints $x \in [-6,6]$ and $y \in [-4,4]$
    \item pitch angle constraint $|\theta| \leq \theta_{\max} = \pi/3$ radians
    \item hazard distance constraint: distance to closest hazard is greater than max hazard radii plus a small buffer to ensure no collision
\end{enumerate}
Initial state is sampled from $x \in [-4.75, -4.25]$, $y \in [1.75,2.25]$, with zero velocity and orientation.
Target position is at $(5.0,-1.5)$ with success radius $0.3$m.
Reward function encourages progress towards the goal: 
\(r = \Delta d-0.001 ||a||^2 - 0.01 \mathbb{I}_{|\theta|>\pi/6}\)
where $\Delta d$ is reduction in goal distance,
second term punishes large control effort,
and third term punishes large non-vertical angles.
A terminal reward of $10$ is given when agent reaches goal.
Cost function returns $1$ for any safety constraint violation and $0$ otherwise.

\begin{table*}[t]
\begin{tabular}{c|c}
\textbf{Hyperparameters}                            & Settings                                                          \\ \hline
Network Architecture                                & MLP                                                               \\
Hidden Layer Activation Function                    & tanh                                                              \\
Actor/Critic Output Activation Function             & tanh/linear                                                       \\
Optimizer                                           & Adam                                                              \\
Discount factor $\gamma$                            & $0.98$                                                            \\
GAE lambda                                          & $0.97$                                                            \\
Clip ratio                                          & $0.2$                                                             \\
Critic learning rate                                & Linear Decay $\num{1e-3} \rightarrow 0$                                 \\
Actor learning rate                                 & Cosine Decay $\num{8e-4} \rightarrow \num{4e-5}$                              \\
Dynamics learning rate                              & Linear Decay $\num{8e-4} \rightarrow 0$                                 \\
Seeds per algorithm per experiment                  & $10$                                                              \\ \hline
\textbf{ReCORS specific parameters}                 &                                                                   \\
Uncertainty Output Activation Function              & sigmoid                                                           \\
Uncertainty learning rate                           & Cosine Decay $\num{8e-4} \rightarrow \num{4e-5}$                                \\ 
$w_{\text{max}}/w_{\text{improve}}$                 & $0.5/1.0$  \\
$w_1/w_2/w_3$                                       & $1.0/0.5/1.0$ \\
\hline
\textbf{Experiment specific settings}               &                                     \textit{(Hidden Layer Size/Total Env Interactions)}                                                                 \\
Cartpole                                            & $(12,12) \quad / \quad 1e6$                                       \\
LaneFollow                                          & $(12,12) \quad / \quad 6e4$                                     \\
2D Quadrotor                                        & $(256,256,256) \quad / \quad 3e5$ \\
3D Quadrotor                                        & $(256,256,256) \quad / \quad 3e5$                                 \\
2D Quadrotor nonlinear                              & $(256,256) \quad / \quad 3e5$  \\
CarGoal \& HalfCheetah                              & $(256,256) \quad / \quad 9e6$  
\\
\hline
\end{tabular}
\vspace{5pt}
\caption{Hyperparameters and environment settings.}
\label{tab:hyperparameters_and_settings}
\end{table*}

\paragraph{3D Quadrotor:}
For our 3D Quadrotor environment, we follow the settings in~\cite{dai2021lyapunov} for the system dynamics:
\begin{align*}
    \ddot{\mathbf{p}} &= g\mathbf{e}_3 + \frac{1}{m}\mathbf{R}\mathbf{f} \\
    \dot{\mathbf{R}} &= \mathbf{R}\hat{\boldsymbol{\omega}} \\
    \dot{\boldsymbol{\omega}} &= \mathbf{I}^{-1}(\boldsymbol{\tau} - \boldsymbol{\omega} \times \mathbf{I}\boldsymbol{\omega})
\end{align*}
where
$\mathbf{p}=[x,y,z]^T$ is position,
$\mathbf{R}$ is the rotation matrix,
$\boldsymbol{\omega}$ is the angular velocity in body frame,
$\mathbf{f}=[0,0,\Sigma_{i=1}^4,f_i]^T$ is the total thrust vector,
$\boldsymbol{\tau}$ is the torque vector from motor thrusts,
$m=0.468$kg is drone mass,
$\mathbf{I}=\text{diag}(4.9,4.9,4.8)\times 10^{-3}$kg$\cdot$m$^2$ is the inertia matrix,
and $g=9.81$m/s$^2$ is gravitational acceleration. 
Discrete timestep $\Delta t$ of the system is set to $0.02$ seconds.

The agent has $20$-dimensional observation space with a $12$-dimensional state vector:
$(x,y,z,\phi,\theta,\psi,\dot{x},\dot{y},\dot{z},\omega_x,\omega_y,\omega_z)$ representing position, orientation, linear velocities, and angular velocities.
The other $8$ dimensions in the observation space are the goal compass direction and distance and closest hazard compass direction and distance.
The objective is to navigate to a target location while avoiding static obstacles.
Safety constraints are
\begin{enumerate}
    \item position constraints $x,y,z \in [-6,6]$
    \item attitude limits $|\phi|, |\theta| \leq \pi/2.5$ radians
    \item hazard distance constraint: distance to closest hazard is greater than max hazard radii plus a small buffer
\end{enumerate}
Initial state is sampled from $x\in[-4.75,-4.25]$, $y\in[-2.25,-1.75]$, $z\in[0.75,1.25]$ with zero velocity and orientation.
Target goal position is at $(4.5,4.5,3.0)$ with success radius $0.4$.
Reward function encourages progress towards goal: 
\(
r=\Delta d - 0.005||a||^2 -0.05(\phi^2+\theta^2)
\)
plus a terminal reward of $50$ upon reaching goal position.
Cost function gives $1$ for any safety constraint violation and $0$ otherwise.

\paragraph{2D Quadrotor with Nonlinear Safe Distance Constraint:}
the system dynamics and settings are consistent with those documented in the 2D Quadrotor environment.
We modified the agent's observation space to introduce the nonlinear safe distance constraint.
Specifically, the agent still has the 6-dimensional state and a 3-dimensional vector corresponding to goal location in its observation space. 
However, instead of the compass direction and distance to its closest obstacle, it now observes its position. 
Furthermore, the obstacles in this environments are nonstationary.
The nonlinear safe distance constraint is formally defined as:
let $p_a=(x,y)$ denote the agent's horizontal and vertical positions, and $p_o=(x',y')$ denotes those of the obstacle that is closest to the agent at the current step. Then the constraint is
\(
||p_a-p_o||_2 \geq d_{\text{safe}}
\)
where we use $d_{\text{safe}}=0.5$ in our experiments.

\paragraph{Safety Gymnasium:} 
We consider two environments in Safety Gymnasium~\cite{ji2023safety}: CarGoal and HalfCheetah.
For CarGoal, the car agent has 72 dimensional observation space that includes accelerometer, gyro, magnetometer, velocimeter, angles, angle vecolicites, and lidar readings of the goal and hazards. 
The agent's goal is to reach the goal location while avoiding both hazards spaces and fragile objects. 
For HalfCheetah, the agent has 17 dimensional observation space that includes joint angles and their velocities as well as agent's location along $z$-axis and velocities along $x,y$ plane. 
The agent's goal is to move as quickly as possible while adhering to velocity constraint. 


\subsection{Hyperparameters/Other Settings}

Table~\ref{tab:hyperparameters_and_settings} contains all relevant hyperparameters and settings for our experiments.
For Safety Gymnaiusm environments CarGoal and HalfCheetah, we set number of vector environments to $30$ to accelerate explorations. 

\textbf{Compute Resources: }
Our code runs on NVIDIA GeForce RTX 3090 GPU.

\end{document}